%% file: anonymous-submission-latex-2026.tex
\documentclass[letterpaper]{article} 
\usepackage[submission]{aaai2026}  
\usepackage{times}  
\usepackage{helvet}  
\usepackage{courier}  
\usepackage[hyphens]{url}  
\usepackage{graphicx} 
\usepackage{natbib}  
\usepackage{caption} 
\usepackage{algorithm}
\usepackage[noend]{algorithmic}
\usepackage{amsthm}

\usepackage{newfloat}
\usepackage{listings}
\usepackage{booktabs}
\usepackage{sanghack}

\DeclareCaptionStyle{ruled}{labelfont=normalfont,labelsep=colon,strut=off} 
\floatstyle{ruled}
\newfloat{listing}{tb}{lst}{}
\floatname{listing}{Listing}
\title{Towards Causal Representation Learning with Observable Sources as Auxiliaries}
\author{
    Kwonho Kim\textsuperscript{\rm1},
    Heejeong Nam\textsuperscript{\rm2},
    Inwoo Hwang\textsuperscript{\rm3},
    Sanghack Lee\textsuperscript{\rm1}
}
\affiliations{
    \textsuperscript{\rm 1} Seoul National University\\
    \textsuperscript{\rm 2} Brown University\\
    \textsuperscript{\rm 3} Columbia University

    rlarnjsgh99@snu.ac.kr,
    Heejeong\_nam@brown.edu,
    ih2455@columbia.edu,
    sanghack@snu.ac.kr
}

\usepackage{bibentry}

\begin{document}

\maketitle
\begin{abstract}
\input{sections/AAAI_2025/00-abs}

\end{abstract}

\section{Introduction}
\label{sec:intro}

\input{sections/AAAI_2025/01-intro}

\section{Related Work}
\label{sec:related-work}

\input{sections/AAAI_2025/02-related-work}

\section{Preliminary}
\label{sec:preliminary}

\input{sections/AAAI_2025/03-pre-hazel}

\section{Method}
\label{sec:method}
\input{sections/AAAI_2025/04-method-hazel}

\section{Experiment}
\label{sec:experiment}
\input{sections/AAAI_2025/05-exp}

\section{Conclusion}
\label{sec:conclusion}
\input{sections/AAAI_2025/06-conclusion}

\subsubsection*{Acknowledgements}
\input{sections/AAAI_2025/07-acknowledgement}

\bibliography{reference}
\onecolumn

\appendix

\input{sections/AAAI_2025/90-appendix}

\end{document}

%% file: sections/AAAI_2025/00-abs.tex
Causal representation learning seeks to recover latent factors that generate observational data through a mixing function. Needing assumptions on latent structures or relationships to achieve identifiability in general, prior works often build upon conditional independence given known auxiliary variables. However, prior frameworks limit the scope of auxiliary variables to be external to the mixing function. Yet, in some cases, system-driving latent factors can be easily observed or extracted from data, possibly facilitating identification.
In this paper, we introduce a framework of observable sources being auxiliaries, serving as effective conditioning variables. Our main results show that one can identify entire latent variables up to subspace-wise transformations and permutations using volume-preserving encoders. Moreover, when multiple known auxiliary variables are available, we offer a variable-selection scheme to choose those that maximize recoverability of the latent factors given knowledge of the latent causal graph. 
Finally, we demonstrate the effectiveness of our framework through experiments on synthetic graph and image data, thereby extending the boundaries of current approaches.

%% file: sections/AAAI_2025/01-intro.tex
Understanding the underlying generative process of observations is crucial for scientific discovery. In this context, causal representation learning (CRL) \citep{9363924}, including nonlinear independent component analysis (ICA) \citep{hyvarinen2009independent}, aims to recover latent variables from observed data. This approach holds great promise for applications in areas such as healthcare \citep{sanchez2022causal}, climate science \citep{yao2024marrying}, and recommendation \citep{wang2022causal, wang2024causally, yang2024disentangled}, as understanding the causal mechanisms can lead to better interpretability and improved generalization to new settings.

However, causal representation learning faces challenges in the absence of appropriate supervision or inductive biases, remaining vulnerable to infinitely many spurious solutions \citep{10.1016/S0893-6080(98)00140-3, 47692}.
To ensure identifiability of the generative process, an established line of research introduces the assumption that the data-governing latent sources are conditionally independent given the auxiliary variables $\rvu$ \citep{da26a48db8f94dc4b880adcaae51a28e, pmlr-v108-khemakhem20a} as illustrated in Fig.~\ref{fig1:graph_a}.

Although this assumption enables identifiability, its key insight lies in encouraging the model to disregard dependencies among the underlying sources by conditioning. Consequently, the conditioning (or auxiliary) variable should be positioned outside the latent space that the model aims to recover. This imposes two strong constraints for applicability in real-world settings: (1) access to external labels, and (2) these labels should render all latent sources conditionally independent. While existing works \citep{lu2022invariant, zheng2023generalizing} provide alternative ways to bypass such conditional independence assumptions, the case of considering a known latent variable has not yet been explored.

However, in some cases, some latent variables can be readily observed or inferred from the available data. This could potentially facilitate the identification of unobserved latent variables by providing direct or partial insights into the generative structure, thereby reducing the overall complexity involved in recovering the true latent representations. One illustrative case is that of a robotic arm carrying out a manipulation task. The underlying latent sources may correspond to various physical parameters, such as the joint angles, torques, or the forces exerted by the arm, while the observed variables could consist of camera images capturing the robot's movements. In this context, we  can treat arm angle information directly extracted from image data as observable sources, which provide only partial information about the true latent variables governing the system.



Moreover, scientific systems like robotic arms governed by physical laws can often be represented using causal graphs \citep{10.5555/3023638.3023683, baumann2022identifying}. In such cases, the conditional independence relations implied by the graph (via d-separation) can reveal which subsets of latent variables are identifiable. For example, fixing the angle of a specific joint in a robotic arm can render the movements of the joints before and after it conditionally independent.

When considering the causal graph as more general data-generating process, the conditional independence between the latents can vary depending on which variables are conditioned. Thus, selecting proper auxiliary variables can determine the degree of identifiability. However, this topic, i.e., \textit{how to exploit/select auxiliary variables leveraging graphical information}, has remained unexplored in recent studies.

To this end, we propose a new framework in which one or more true latent variables are assumed to be known or observable from the data. Sec. 2 reviews how prior studies have attempted to bypass strong conditional independence assumptions. In Sec. 3, we formalize the problem, introduce the notations, and specify the setting. Sec. 4 investigates the extent to which recoverability and identifiability can be achieved. Our main result is that the system can be identified at most up to subspace-wise transformations and permutations when the Jacobian matrix of the mixing function is properly designed to have a determinant with absolute value equal to one at every point. In the case where two or more variables are known, we offer a variable-selection scheme to choose those that maximize the recoverability of the latent factors. 
In Sec. 5, we present a volume-preserving, flow-based model that adheres to this condition. We validate our approach across various experimental settings, demonstrating that the representation effectively disentangles according to the conditional independence structure of the latent graph.


    
    

%% file: sections/AAAI_2025/02-related-work.tex
One of the key obstacles in CRL is the dependence among latent sources induced by underlying causal mechanisms. It directly violates the assumption of conditionally independent sources, which underlies the identifiability of many nonlinear ICA approaches that rely on conditionally factorized priors \citep{pmlr-v108-khemakhem20a}. To address this issue, several works explicitly incorporate a known or assumed causal graph over the latent variables to model source dependencies. For example, \citet{yang2021causalvae} (CausalVAE) propose a structured variational autoencoder where the latent variables follow a predefined causal DAG, enabling do-interventions in the latent space. Similarly, \citet{pan2024counterfactual} (ANCM) handle non-Markovian generative processes by modeling image generation with an augmented causal graph that captures temporally entangled latent factors. While these methods provide a framework for incorporating causal structure into representation learning, they operate under a fully supervised setting, assuming access to structured semantic labels or ground-truth causal factors. Moreover, they are primarily focused on image generation and counterfactual editing tasks, rather than the general identifiability or recovery of latent sources from more weakly supervised or observational data. 

To achieve identifiability under such dependencies, many methods rely on interventional data which can be impractical in real-world settings \citep{lippe2023causal, liang2023causal, ali2024crid}. In particular, \citet{liang2023causal} (CauCa) assumes a Markovian graph and leverages interventions for identifiability, while \citet{ali2024crid} (CRID) handles more general non-Markovian settings by explicitly modeling unobserved confounders. Both of CauCA and CRID share with our approach the use of causal graph to guide recovery, suggesting that our method could be extended to non-Markovian settings in future work.

As an alternative, recent efforts have aimed to prove identifiability from observational data alone. \citet{yao2024multiview} introduce a method based on block-identifiability \citep{gelgen2021selfsupervised}, which extracts shared latent variables from multiple views using contrastive learning and entropy regularization. \citet{zhang2024causal} show that structural sparsity among the sources enables identifiability without any explicit causal graph. While these works relax assumptions on data collection, they rely on indirect structural constraints. In contrast, we investigate how to select or exploit observed sources as auxiliary variables under a known causal structure to recover latent sources. This approach retains the strengths of causal modeling while improving recoverability in settings where full interventions or disentangled views are unavailable.

%% file: sections/AAAI_2025/03-pre-hazel.tex
To frame our problem setting, we denote random variables or vectors by uppercase letters and their assignments by lowercase throughout this paper. Bold letters represent a set of random variables or random variables which is not a singleton. We write $[d]$ to denote the set $\{1, 2, \cdots, d\}$.

\subsection{Data Generating Process}
We define $\*{x} \in \mathbb{R}^m$ as an observation (e.g., image) which are generated from latent sources $\rvz \in \mathbb{R}^n$ as follow:
\begin{equation}\label{eqn1:generation}
    \rvx=g(\rvz).
\end{equation}
where $g$, a mixing function, is an arbitrary invertible and smooth nonlinear function in the sense that its second-order derivatives exist.
By adopting a Bayesian network, we represent a data-generating process regarding latent sources as
\begin{equation}\label{eqn2:source dep}
    z_i = f_i(\Pa{z_i}, \epsilon_i), \quad \epsilon_i \sim p_{\epsilon_i}, 
\end{equation}
for all $i \in [n]$ 
where $\Pa{\cdot}$ represents parent nodes on a known causal graph $\mathcal{G}$ consisting of nodes $V$ and edges $E$.

\input{figs/fig1-hazel}

\subsection{Leveraging Conditional Independence}

The primary goal of CRL is to recover inverse mapping $g^{-1}$ and the true latent sources $\rvz=(z_1, \cdots, z_n)$ from observations.
The simplest setup might consider fully independent sources as in nonlinear ICA with latent probabilities of Eq.~\ref{eq:source-factorization-independent}. 
\begin{equation}\label{eq:source-factorization-independent}
p(\rvz)=\prod_{i=1}^n p(z_i).
\end{equation}

However, with only i.i.d. samples, this problem is provably unidentifiable: one cannot uniquely map observations to mutually independent sources under that assumption alone \citep{10.1016/S0893-6080(98)00140-3}. To ensure identifiability, many researchers have relied on conditional independence assumption given observable variables. 

Fig.~\ref{fig1:graph_a} depicts the setup of CRL with auxiliary variables: the latent sources \(\mathbf{z}=(z_1,z_2,z_3,z_4)\) are mutually independent by conditioning an observable auxiliary variable \(u\), i.e.
\begin{equation}\label{eq:source-factorization-independent-auxiliary}
p(\rvz\mid \rvu)=\prod_{i=1}^n p(z_i\mid \rvu),
\end{equation}
where $\rvu$ can be a class label, time index, or historical information \citep{pmlr-v89-hyvarinen19a}. 

Fig.~\ref{fig1:graph_b} similarly assumes that the observable variable is given as auxiliary, lying outside of the data generating process. However it falls under Independent Subspace Analysis (ISA), which focuses on recovering only up to the subspace that becomes independent under conditioning. Rather than identifying all latent sources, ISA relaxes the assumption. This can be written as
\begin{equation}\label{eq:subspace-factorization-isa}
p(\rvz \mid \rvu) = \prod_{i=1}^{d} p(\rvz_{c_i} \mid \rvu),
\end{equation}
where $\{\rvz_{c_i}\}_{i=1}^d$ is a partition of the latent sources such that $\cup_{i=1}^{d} \rvz_{c_i} = \{z_1,\dots,z_n\} \setminus \{z_o\}$, and each $\rvz_{c_i}$ denotes a conditionally dependent subspace. 

Fig.~\ref{fig1:graph_c} and Fig.~\ref{fig1:graph_d} depict our setting. Unlike Fig.~\ref{fig1:graph_a} and Fig.~\ref{fig1:graph_b}, the observable variables directly participate in the data generating process. These cases also follow the ISA setting, where full conditional independence among latent sources is not required. Concretely, we treat a subset of the observed latent sources $\rvz_o \subset \rvz$  as conditioning variables that directly influence the mixing. Their generative process is then captured by DAG \(\mathcal{G}\), which encodes arbitrary dependencies among all sources.

Accordingly, Eq. \ref{eq:subspace-factorization-isa} is reformulated in our setting as:

\begin{equation}\label{eq:source-factorization-partial}
p_{\rvz_{o^-} | \rvz_o}(\rvz_{o^-} | \rvz_o) =  \prod_{j=1}^{d} p_{\rvz_{c_j}| \rvz_o}(\rvz_{c_j}| \rvz_o),
\end{equation}
where $\rvz_o$ is observed sources and $\rvz_{o^-}$ is unobserved sources.

Note that we exclude the degenerate case $\mathbf{z}_o = \emptyset$ (i.e., no auxiliary variables), since without any conditioning the problem falls into a different regime that demands extra assumptions for identifiability, such as structural sparsity \citep{zheng2023generalizing}. We defer the discussion of how to handle cases where $\mathbf{z}_o$ contains more than one element to Section~\ref{sec:method-selection}.

\subsection{Problem Formulation}\label{pre:problem}

Suppose a data generating process in Eq.~\ref{eqn1:generation} and Eq.~\ref{eqn2:source dep}. 
Our goal is to establish the identifiability of the independent latent sources (i.e., $\rvz_{c_i}$) up to certain subspace-wise invertible transformation and permutation, given the observations $\rvx$, observable sources $\rvz_o (\subseteq \rvz)$, and the latent Bayesian network $\mathcal{G}$ which encodes the conditional independence relationships between the latent sources as shown in Fig.~\ref{fig1}.

We assume that the latent Bayesian network $\mathcal{G}$ is known, which allows us to leverage diverse conditional independence relationships between the sources.
Importantly, the partition of the latent sources into subspaces $\rvz_{o^-}=\cup_i \rvz_{c_i}$ determines the \textit{degree} of the identifiability we could achieve (Thm. 4.3 of \citet{zheng2023generalizing}).
Therefore, it is crucial to capture the proper observable sources $\rvz_u\subseteq \rvz_o$ that entails fine-grained subspaces $\rvz_{c_j}$ mutually independent to each other conditioned on $\rvz_u$.

\paragraph{Motivating Examples}
While in simple cases such as Fig.~\ref{fig1:graph_c}, it is straightforward to find the most fine-grained conditionally independent groups, i.e. ${\{z_1\},\,\{z_2\},\, \{z_4,z_5\}}$. 
However, considering the case where the multiple sources are observed as Fig.~\ref{fig1:graph_d}; if all observed sources are conditioned, $z_1$ and $z_3$ cannot be disentangled. $z_2$ should not be conditioned to satisfy conditional independence.

If the number of sources or observed sources increases, finding auxiliary variables that make the latent grouping fine-grained can become computationally challenging. It is true that a problem of which node to condition on typically arises in collider structures, and even the same node can act as a confounder for one group of nodes and as a collider for another. In such case, it becomes necessary to carefully choose which nodes to condition on in order to achieve optimal results.

%% file: figs/fig1-hazel.tex
\begin{figure}[t]
  \centering
  \begin{subfigure}[b]{0.21\textwidth}\centering
    \captionsetup{justification=centering}
    \begin{tikzpicture}[x=.9cm,y=6mm,scale=0.75,transform shape,font=\normalsize]
      \node[RR, minimum size=7mm][fill=betterblue!50] (u1) at (0,2) {$u$};
      \node[RR, minimum size=7mm] (z1) at (-1.5,0) {$z_1$};
      \node[RR, minimum size=7mm] (z2) at (-0.5,0) {$z_2$};
      \node[RR, minimum size=7mm] (z3) at (0.5,0) {$z_3$};
      \node[RR, minimum size=7mm] (z4) at (1.5,0) {$z_5$};
      \draw[->] (u1) -- (z1);
      \draw[->] (u1) -- (z2);
      \draw[->] (u1) -- (z3);
      \draw[->] (u1) -- (z4);
    \end{tikzpicture}
    \caption{Single auxiliary variable for ICA}
    \label{fig1:graph_a}
  \end{subfigure}
  \hspace{1em}
  \begin{subfigure}[b]{0.21\textwidth}\centering
    \captionsetup{justification=centering}
    \begin{tikzpicture}[x=.9cm,y=6mm,scale=0.75,transform shape,font=\normalsize]
      \node[RR, minimum size=7mm][fill=betterblue!50] (u1) at (0,2) {$u$};
      \node[RR, minimum size=7mm] (z1) at (-1.5,0) {$z_1$};
      \node[RR, minimum size=7mm] (z2) at (-0.5,0) {$z_2$};
      \node[RR, minimum size=7mm] (z3) at (0.5,0) {$z_3$};
      \node[RR, minimum size=7mm] (z4) at (2,0) {$z_5$};
      \draw[->] (u1) -- (z1);
      \draw[->] (u1) -- (z2);
      \draw[->] (u1) -- (z3);
      \draw[->] (z3) -- (z4);
    \end{tikzpicture}
    \caption{Single auxiliary variable for ISA}
    \label{fig1:graph_b}
  \end{subfigure}

  \vspace{1em} 

  \begin{subfigure}[b]{0.21\textwidth}\centering
    \captionsetup{justification=centering}
    \begin{tikzpicture}[x=.9cm,y=6mm,scale=0.75,transform shape,font=\normalsize]
      \node[RR, minimum size=7mm] (u1_c) at (8,2) {$z_3$};
      \node[RR, minimum size=7mm] (z1_c) at (6.5,0) {$z_1$};
      \node[RR, minimum size=7mm][fill=betterblue!50] (z2_c) at (8,0) {$z_5$};
      \node[RR, minimum size=7mm] (z3_c) at (9.5,0) {$z_4$};
      \node[RR, minimum size=7mm] (z5_c) at (6.5,2) {$z_2$};
      \draw[->] (u1_c) -- (z2_c);
      \draw[->] (u1_c) -- (z3_c);
      \draw[->] (z2_c) -- (z1_c);
      \draw[->] (z2_c) -- (z3_c);
      \draw[->] (z2_c) -- (z5_c);
    \end{tikzpicture}
    \caption{\textit{Ours. }Single observable source for ISA}
    \label{fig1:graph_c}
  \end{subfigure}
  \hspace{1em}
  \begin{subfigure}[b]{0.21\textwidth}\centering
    \captionsetup{justification=centering}
    \begin{tikzpicture}[x=.9cm,y=6mm,scale=0.75,transform shape,font=\normalsize]
      \node[RR, minimum size=7mm][fill=betterblue!50] (z4_c) at (8,2) {$z_4$};
      \node[RR, minimum size=7mm] (z1_c) at (6.5,0) {$z_1$};
      \node[RR, minimum size=7mm][fill=betterblue!50] (z2_c) at (8,0) {$z_2$};
      \node[RR, minimum size=7mm] (z3_c) at (9.5,0) {$z_3$};
      \draw[->] (z4_c) -- (z3_c);
      \draw[->] (z1_c) -- (u1_c);
      \draw[<-] (z2_c) -- (z1_c);
      \draw[<-] (z2_c) -- (z3_c);
    \end{tikzpicture}
    \caption{\textit{Ours. } Multiple observable sources for ISA}
    \label{fig1:graph_d}
  \end{subfigure}

  \caption{Latent mechanisms in different setups leveraging conditional independence. Blue nodes denote observables.}
  \label{fig1}
\end{figure}
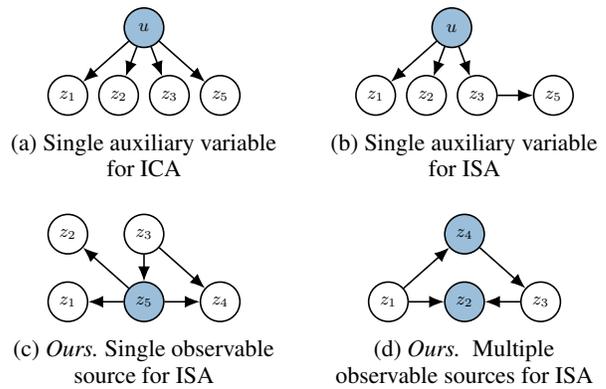

%% file: sections/AAAI_2025/04-method-hazel.tex
In this section, we establish identifiability in the presence of observable sources (Sec.~\ref{sec:method-identifiability}). Based on conditions for identifiability, we introduce a framework with a graphical criterion to effectively leverage auxiliary variables that makes the conditionally independent latents more fine-grained (Sec.~\ref{sec:method-selection}) and method to recover unobserved latents (Sec.~\ref{sec:method-architecture}).

\subsection{Identifiability}
\label{sec:method-identifiability}

To deal with problems that the observed sources $\rvz_o$ are included in the mixing function, we assume that the mixing function is constrained to a specific form as \citet{yang2022nonlinear}.
We adopt the proof of \citet{zheng2023generalizing} for identifiability with dependent sources.

\begin{proposition}\label{prop:partial}
    (Identifiability with observable sources). Suppose the following assumptions hold:
    
        
        
    \begin{enumerate}
        \item The observed data and sources are generated from Eq.~\ref{eqn1:generation} and Eq.~\ref{eq:source-factorization-partial}.
        
        \item The mixing function $g$ is volume-preserving, i.e., $|\det(\mathbf{J}_g(\rvz))| = 1$.
        
        \item For every value of $\rvz_{o^-}$, there exist $2d$ values of $\rvz_o$ such that the $2d$ vectors 
        $\rvw(\rvz_{o^-}, \rvz_{o_i})$ are linearly independent, where each vector is defined as
        \begin{align*}
            \rvw(\rvz_{o^-}, \rvz_{o_i}) = \big( 
            &\rvv(\rvz_{c_1}, \rvz_{o_i}), \dots, \rvv(\rvz_{c_d}, \rvz_{o_i}), \\
            &\rvv'(\rvz_{c_1}, \rvz_{o_i}), \dots, \rvv'(\rvz_{c_d}, \rvz_{o_i}) 
            \big).
        \end{align*}
    \end{enumerate}
    where {\footnotesize
        \begin{align*}
        \rvv(\rvz_{c_j}, \rvz_{o_i}) &= \left(\! \frac{\partial \log p(\rvz_{c_j} | \rvz_{o_i})}{\partial z_{c_j}^{(l)}}, \dots, \frac{\partial \log p(\rvz_{c_j} | \rvz_{o_i})}{\partial z_{c_j}^{(h)}} \!\right), \\
        \rvv'(\rvz_{c_j}, \rvz_{o_i}) &= \left(\! \frac{\partial^2\! \log p(\rvz_{c_j} | \rvz_{o_i})}{\partial (z_{c_j}^{(l)})^2},  \dots, \frac{\partial^2\! \log p(\rvz_{c_j} | \rvz_{o_i})}{\partial (z_{c_j}^{(h)})^2} \!\right)
        \end{align*}} and $\rvz_{c_j} = (z_{c_{j}^{(l)}},\dots,z_{c_{j}^{(h)}})$.

    Then all the components of $\rvz_{o^-}$ (i.e., $\rvz_{c_i}$ where $c_i \in \{c_1,\dots,c_d\}$) is identifiable up to a subspace-wise invertible transformation and a subspace-wise permutation.
\end{proposition}

Most prior works on CRL achieve identifiability by assuming the mixing function is fixed across environments. Under a common invertible mixing $g$, one can write the log-likelihood of the data under two domains and subtract them, causing the Jacobian log‐determinant terms to cancel (since the same $g$ applies in both cases). In such settings the latent distributions change across domains while $g$ remains invariant, so the log-determinant terms disappear when differencing log-likelihoods. 

In our setting, by contrast, the observed source (domain label) is used to index the mixing function, so $g$ varies with the source. As a result the usual cancellation does not occur and the standard identifiability proof breaks down. To address the problem, we constrain the mixing to be volume-preserving (i.e., $|\text{det}(\mathbf{J}_g(\rvz))| = 0$ everywhere). With the volume-preserving assumption on the mixing function, the Jacobian determinant remains constant at 1, making the log-determinant term equal to zero (the proof in Appendix~\ref{appx:id}).

\begin{remark} \label{remark:generalization} 
(Connection between auxiliary variables and observable sources). 
We can consider existing approaches with auxiliary variables (e.g. Fig.~\ref{fig1:graph_a}) as the case that observed sources $\rvz_o$ do not have edges into the observations $\rvx$. Although the assumption we impose on the mixing function to ensure identifiability is quite strict, it can cover the identifiability of the existing settings. See Proposition~\ref{prop:auxiliary} in Appendix~\ref{appx:id}. 
\end{remark}



\subsection{Selecting Observables for Conditioning}
\label{sec:method-selection}

According to the Proposition~\ref{prop:partial}, the conditional independence determines the number of recoverable sources in the identifiability of latent variables and our goal is first to capture mutually independent groups of nodes given observable sources and the known causal graph.
However, a naive approach of leveraging all observed sources might not capture conditional independence relationships, i.e., $z_1 \not \Perp z_3 \mid z_2, z_4$ in Fig.~\ref{fig1:graph_d}.
It is necessary to capture a proper subset of observed sources that entails the most fine-grained groups of mutually independent sources, and ultimately, leads to the most granular identifiability.

Formally, we aim to discover a conditional independence structure that partition $\rvz_{o^-}$ into the most fine-grained subgroups such that:
\begin{equation}\label{eq:conditional-independence-cluster}
    \rvz_{c_i} \Perp \rvz_{c_j} \mid \rvz_u, \quad \text{for all } i\neq j,
\end{equation}
where $\rvz_u\subseteq \rvz_o$, $\cup_i \rvz_{c_i} \subseteq \rvz\setminus\rvz_o$, and $\rvz_{c_i}\cap \rvz_{c_j} = \emptyset$ for all $i\neq j$.
Importantly, satisfying a fine-grained conditional independence condition enables the identification of a greater number of latent variables. This ensures a more precise disentanglement of the underlying causal structure, improving recoverability and manipulability of the true latent factors.

We propose a strategy that selects the most fine-grained conditionally independent groups of the latents with the minimum set of observed sources in Algorithm~\ref{alg:selection}. The algorithm initializes a candidate set by including only nodes that act as confounders and excluding those that act solely as colliders, in order to account for nodes that may serve as both. The \textit{Partition} algorithm counts the number of groups that satisfy conditional independence by running \textit{Bayes-ball} \citep{10.5555/2074094.2074151} algorithm repeatedly. Finally, the algorithm outputs the conditioning set that results in the largest number of groups, i.e., the most fine-grained partitioning.

\input{algorithm/alg3_sc}


\paragraph{Example}
Consider the latent graph in Fig.~\ref{fig1:graph_d}. Observed set $O = \{z_2, z_4\}$.
We will iterate all the subsets of $O$, i.e., $\{z_2\}, \{z_4\}, \{z_2,z_4\}$.\footnote{$\emptyset$ cannot be considered due to the identifiability condition.}  Firstly, with conditioning set $\{z_2\}$, the partition process is as follow:
\begin{enumerate}
    \item Started from $z_1$, the \textit{result} contains $z_1$. 
    \item \textit{Bayes-ball} algorithm get input as $G,\{z_2\}$ and \textit{result}.
    \item In the \textit{Bayes-ball}, path from $z_1$ to $z_3$ through $z_2$ cannot be d-separated because $z_2$ works as collider.
    \item The path from $z_1$ to $z_3$ also cannot be d-separated.
    \item The \textit{result} is $ \{\{z_1,z_3\} \}$ except for observed source $z_2$.
\end{enumerate}
With conditioning set $\{z_4\}$, by following same process, the \textit{result} will be $\{ \{z_1\},\{z_3\} \}$. The conditioning set $\{z_2,z_4\}$ makes the result to be $\{ \{z_1,z_3\} \}$. Hence, the selection result will be $\{z_4\}$ for the most fine-grained conditionally independent latents. 

\subsection{Learning to Recover}
\label{sec:method-architecture}
To construct a representation that satisfies the identifiability conditions in \ref{prop:partial}, we enforce volume preservation in the encoder by adopting  General Incompressible-flow Network (GIN) \citep{sorrenson2020disentanglementnonlinearicageneral} as our encoder. In addition to volume preservation, we also impose a graphical constraint via a structural neural network to preserve dependencies among latent variables that are not assumed to be independent, reflecting the known latent causal structure to strengthen disentanglement.

\paragraph{Volume-preservation} While GIN originally optimizes only the log-likelihood of the conditional distribution given the auxiliary variables, we factorize the log-likelihood of the distribution as follows:
\[
 \log p_{\hat{g}^{-1}}(\rvx) = \log p(\hat{\rvz}) =\log p(\rvz_u)+ \sum_i \log p(\hat{\rvz}_{u_{i}^-} \mid \rvz_u),
\]
where $\hat{\rvz}_{u_{i}^-}=\hat{\rvz}\setminus \hat{\rvz}_{u_{i}}$.
By factorizing the log-likelihood of the distribution, we can naturally address the issue that the information from the auxiliary variable is directly entangled with the observations. The preceding term will serve to absorb information about $\rvz_u$ from $\rvx$ while the latter term enforces the components of $\rvz_{u^-}$ to be independent given $\rvz_u$ by modeling them as a multivariate normal distribution with zero off-diagonal elements.

\paragraph{Graphical constraint} Besides, $\hat{\rvz}_{u^-}$ contain the information of sources that are observed but not selected (expressed as $\rvz_n$), i.e., $\hat{\rvz}_{u^-} = \{\hat{\rvz}_{o^-}, \hat{\rvz}_n\}$. We need to keep the relationship between $\hat{\rvz}_n$ and $\hat{\rvz}_{o^-}$ which is not independent (relationship between $z_2$ and ${z_1, z_3}$ in \ref{fig1:graph_d}). 

To deal with this problem, we leverage the structural neural net to enforce the relationship between $\hat{\rvz}_{o^-}$ and $\hat{\rvz}_n$. A structural neural network is designed based on the latent graph $\mathcal{G}$ and not selected label $\rvz_n$.
Specifically, $\rvz_n$ is predicted by arbitrary dimensions of $\hat{\rvz}_{u^-}$ working as parents of $\rvz_n$. 
Since we do not know exactly which dimension of the representation corresponds to which true latent variable, we rely only on the number of parents of $\rvz_n$. For example, in \ref{fig1:graph_d}, true $z_2$ is predicted by the certain dimension of the estimated representation given the other dimensions ($\hat{z}_1, \hat{z}_3$), naturally reflecting the causal structure. The full objective function is:
\begin{align}
\mathcal{L}(\theta) = \mathbb{E}_{(\rvx, \rvz_u, \rvz_n) \in \mathcal{D}} \Big[
    & \log p(\rvz_u) + \sum_i \log p(\hat{\rvz}_{u_{i}^-} \mid \rvz_u) \nonumber \\
    & + \log p(\rvz_n \mid \Pa{\rvz_n}) 
\Big].
\end{align}

%% file: algorithm/alg3_sc.tex
\begin{algorithm}[t]
\footnotesize
    \begin{algorithmic}[1]
    \STATE \textbf{Input}: graph $G$, observed set $O$
    \STATE \textbf{Output}: conditioning set $C$
    \STATE $C \leftarrow \{$nodes acting only as confounders on $G$ $\}$
    \STATE $O \leftarrow O \setminus \{$nodes acting only as colliders on $G$ $\}$
    \STATE $max \gets 0$
    \FOR{each subset $T \subseteq O$}
        \STATE $S \gets Partition(G, T, O)$
        \IF{$|S| > max$ \textbf{or} ($|S| = max$ \textbf{and} $|T| < |C|$)}
            \STATE $max \gets |S|$
            \STATE $C \gets T$
        \ENDIF
    \ENDFOR
    \STATE \textbf{return} $C$
    \end{algorithmic}
\caption{Selecting Observables for Conditioning}\label{alg:selection}
\end{algorithm}

%% file: sections/AAAI_2025/05-exp.tex
We conduct experiments to empirically validate both the effectiveness of the selection procedure and the capability of our proposed architecture in leveraging observable sources.
\input{figs/fig2}
\subsection{Experimental Setup} 
\paragraph{Data} Reflecting the setup of observable sources, we consider synthetic datasets generated from the three graphs in Fig.~\ref{fig1:graph_a}, ~\ref{fig1:graph_c} and ~\ref{fig1}:
$\mathcal{D} = \{ (\rvx^{(i)}, \rvz_o^{(i)})\}_{i=1}^N$,
where $N$ is the sample size and $\rvz_o^{(i)}$ is the observed sources corresponding to the data point $\rvx^{(i)}$. When we run our selection procedure given a graph to choose the best combination of the auxiliary variables, $\rvz_o$ will be partitioned into $\rvz_u$ and $\rvz_n$. 

The data was generated using a linear Structural Causal Model (SCM) where each variable is determined by a linear combination of its parents and an additive noise term:
\begin{equation}
X_i = \sum_{j \in \text{pa}(X_i)} \beta_{ij} X_j + \varepsilon_i,
\end{equation}
where $\beta_{ij}$ are sampled uniformly from $[0.5, 1.0]$, and $\varepsilon_i$ is the additive noise term with coefficient fixed to 1.0.

To further demonstrate the effectiveness of our method on high-dimensional data, we used the Pendulum and modified Flow datasets  \citep{yang2021causalvae}, which consist of structured, systematically sampled image data (Fig.~\ref{fig:causal_graphs}). The implementation details are in Appendix \ref{appx:imple}.

\paragraph{Metrics} After training the proposed method, we measure Disentanglement, Completeness, Informativeness (DCI) metric \citep{eastwood2018a} based on Mean Correlation Coefficient (MCC) matrix which is a widely accepted metric in the literature for measuring the degree of identifiability \citep{da26a48db8f94dc4b880adcaae51a28e}. We assess how well the learned representation aligns with the independence structure of the underlying graph.

Specifically, the MCC matrix is defined as:
\begin{equation}
\mathbf{MCC}_{ij} = \text{corr}(z_i, \hat{z}_j),
\end{equation}
where each entry $ \mathbf{MCC}_{ij}$ represents the Pearson correlation coefficient between the true latent variable \( z_i \) and the estimated $ \hat{z}_j$.  
The optimal permutation $\sigma^*$  is selected to maximize the total correlation, ensuring that each estimated latent variable is matched to the most similar true latent variable.

Based on the computed MCC matrix, we evaluate models with Disentanglement and Completeness among DCI metrics:
\begin{equation}
D = 1 - H(P_{i,\cdot}),
\end{equation}
\begin{equation}
C = 1 - H(P_{\cdot,j}),
\end{equation}
where \( P_{i,j} \) is the value from the MCC matrix, representing the contribution of the estimated latent variable \( \hat{z}_j \) to the true latent factor \( z_i \). The entropy function \( H(\cdot) \) measures the dispersion of importance values across dimensions, ensuring that a lower entropy corresponds to a more structured and disentangled representation.
Disentanglement (\( D \)) quantifies whether each estimated latent variable captures at most one true latent factor, computed by applying row-wise entropy over \( P_{i,\cdot} \). Completeness (\( C \)) assesses whether each true latent factor is captured by a single estimated latent variable, computed via column-wise entropy over \( P_{\cdot,j} \). Since both scores range from 0 to 1, higher values indicate better structured representations with minimal mixing between factors. Further discussions on why MCC alone is insufficient for evaluation are provided in Appendix \ref{apx:discussion}. All the metrics are measured over 20 repetitions.

\subsection{Empirical Results}
\paragraph{Effectiveness of selection}
\begin{figure}
    \centering
    \includegraphics[width=0.4\linewidth]{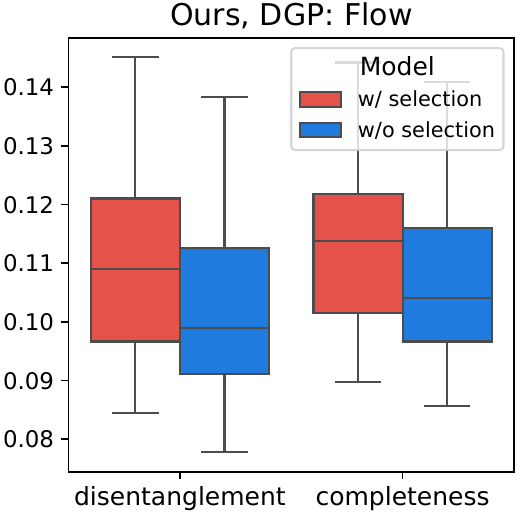}
    \caption{Ablation study on the selection procedure for Ours.}
    \label{fig:ablation-selection}
\end{figure}

We conducted an ablation study on the selection procedure for our architecture. The experiments are based on the data-generating process illustrated in Fig.~\ref{fig:flow}, where the differences in results arise depending on the selection procedure. Fig.~\ref{fig:ablation-selection} shows the change in DCI metric for our model before and after selection. The selection procedure improves disentanglement in the representation as shown in Fig.~\ref{fig:ablation-selection}.
For the graph in Fig.~\ref{fig:flow}, using all observed sources as auxiliary variables without a selection procedure breaks the conditional independence between \textbf{Water Height} and \textbf{Hole}, leading to entangled representations.

\begin{figure}
    \centering
    \includegraphics[width=1\linewidth]{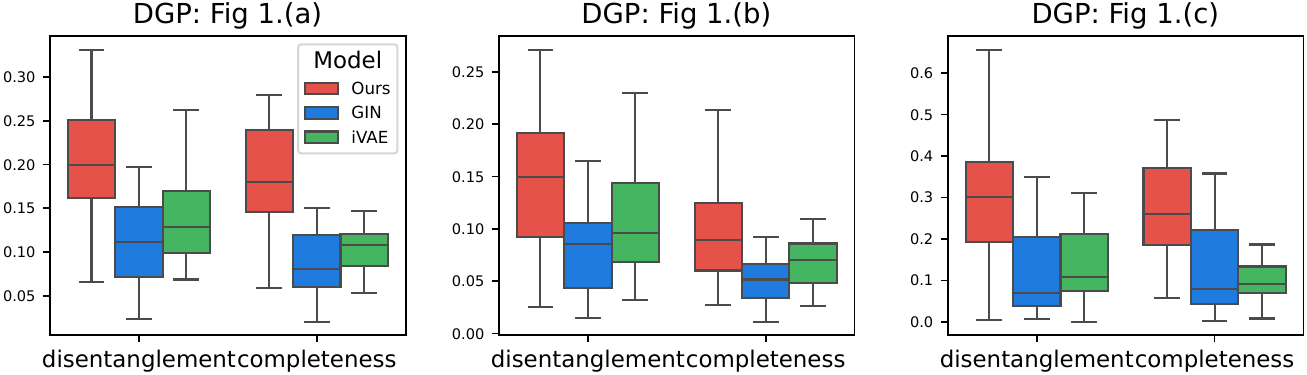}
    \caption{Comparison plot for DCI metric between Ours, GIN, and iVAE.}
    \label{fig:dci-architecture}
\end{figure}

\paragraph{Effectiveness of architecture}
To verify the effectiveness of our architecture, we choose GIN \citep{sorrenson2020disentanglementnonlinearicageneral} and iVAE \citep{pmlr-v108-khemakhem20a} as baseline models. GIN is used as the encoder in our architecture, ensuring the volume-preserving property but not designed to handle observed sources. iVAE is also not designed to handle partially observed sources. Furthermore, it does not impose any constraints on the mixing function and solely relies on a multivariate normal distribution as the prior, ensuring that each latent variable is conditionally factorizable. For a fair comparison, all experiments are conducted using the same auxiliary variables filtered through the selection procedure.

Fig.~\ref{fig:dci-architecture} demonstrates that our proposed method outperforms other approaches in terms of the DCI metric. 
Our proposed method maximizes the likelihood of a conditionally factorizable distribution for the remaining components while simultaneously excluding the information of auxiliary variables mixed with the observation $\rvx$. This prevents spurious correlations in the representation by ensuring that the information of $\rvz_u$, which is related to unobserved latents, does not mix into the representation.

\begin{figure}[t]
    \centering

    \begin{subfigure}{.95\linewidth}
        \centering
        \includegraphics[width=\linewidth]{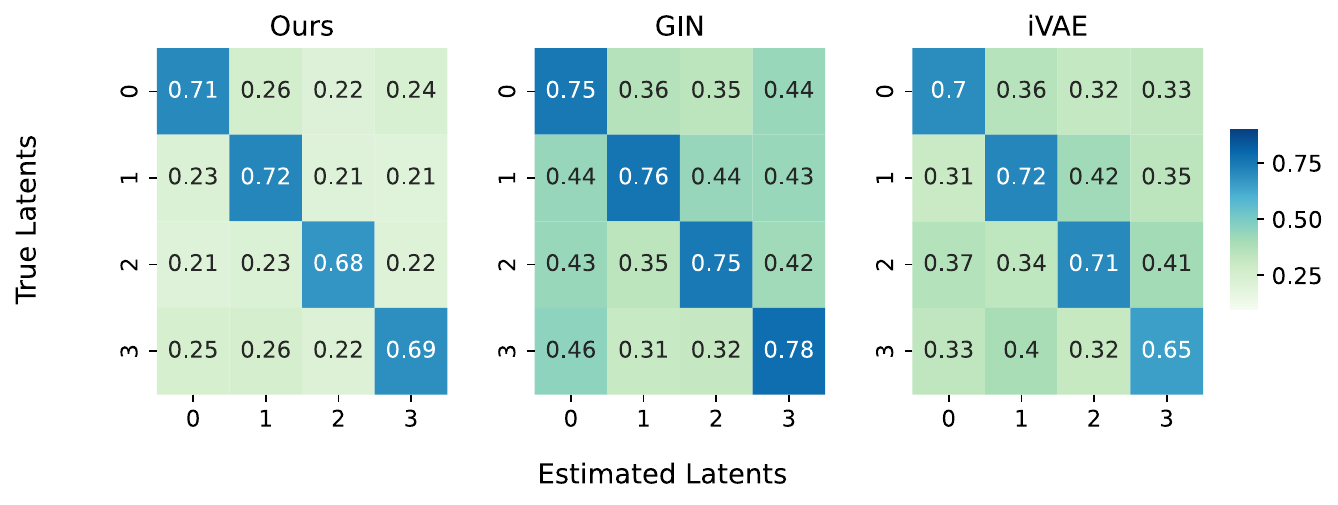}
        \caption{}
        \label{fig:mcc-a-architecture}
    \end{subfigure}

    \begin{subfigure}{.95\linewidth}
        \centering
        \includegraphics[width=\linewidth]{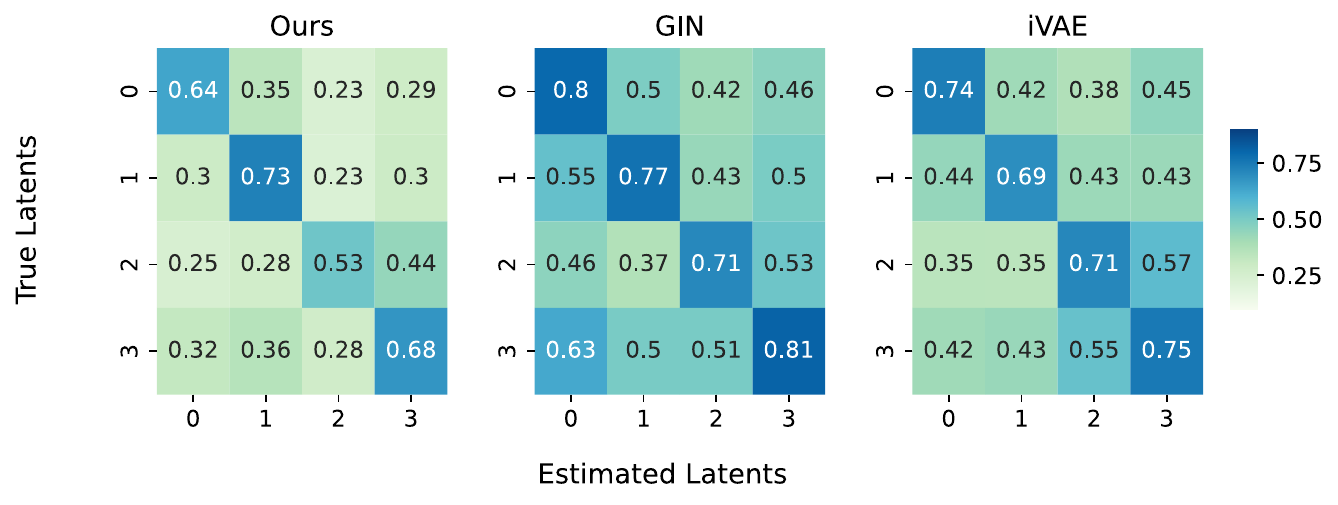}
        \caption{}
        \label{fig:mcc-b-architecture}
    \end{subfigure}

    \begin{subfigure}{.95\linewidth}
        \centering
        \includegraphics[width=\linewidth]{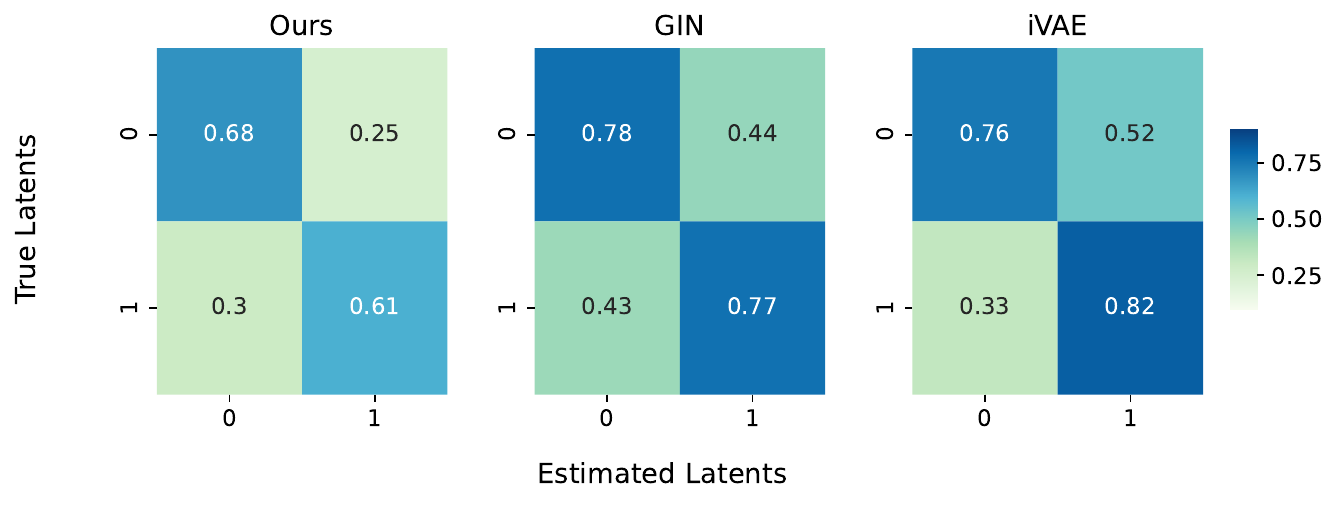}
        \caption{}
        \label{fig:mcc-c-architecture}
    \end{subfigure}

    \caption{Mean correlation matrices of Ours, GIN applying selection, and iVAE applying selection matched with the best permutation for the settings of Fig.~\ref{fig1:graph_a} (top), Fig.~\ref{fig1:graph_c} (middle), and Fig.~\ref{fig1:graph_d} (bottom).}
\end{figure}


For a more detailed examination, we analyzed the results with the MCC matrix. The proposed architecture shows a comparable MCC score (mean of diagonal terms) as GIN and iVAE (Fig.~\ref{fig:mcc-a-architecture} for the DGP of Fig.~\ref{fig1:graph_a}). However, looking at the MCC matrix, we can observe that both GIN and iVAE show high correlations with the other latents besides the true latent, even when matched with the best permutation,  This suggests that manipulating a specific dimension of the representation simultaneously affects other latents, indicating that the representation is not well disentangled. 


Considering the DGP in Fig.~\ref{fig1:graph_c}, the ideal disentanglement is that $\hat{z}_1$, $\hat{z}_2$, ($\hat{z}_3$, and $\hat{z}_4$) are conditionally independent. The result of our architecture for MCC matrix in Fig.~\ref{fig:mcc-b-architecture} represents the almost ideal disentanglement, while the other methods still show entangled results. As the conditional independence in DGP in Fig.~\ref{fig1:graph_c} does not ensure each latent to be identified, but block-identified, the MCC score  might be lower. Even in this case, GIN and iVAE, which do not consider observed sources, show a high MCC score because of spurious correlation. Likewise, on the DGP  (Fig.~\ref{fig1:graph_d}), our architecture yields a disentangled MCC matrix as expected.


\paragraph{High-Dimensional data}
We also conduct the experiments on the Pendulum and Flow datasets \citep{yang2021causalvae}. The images are generated by a latent mechanism shown in Fig.~\ref{fig:causal_graphs}. The images have a size of 4$\times$96$\times$96. For the Flow dataset, the auxiliary variable \textbf{Ball Size} is determined through the selection process. For the Pendulum dataset, all observed latents should be selected as auxiliary variables to ensure the conditional independence of the unobserved latent variables.

As illustrated in Fig.~\ref{fig:dci-highdim}, our proposed method demonstrated performance that is comparable to or superior to other models. Unlike the results on synthetic data, the GIN model exhibited strong performance because its normalizing flow-based architecture is more suitable for handling image data in terms of model capacity. The MCC matrices (Figs.~\ref{fig:flow} and \ref{fig:pendulum}) also show that our method learns disentangled representations for unobserved latent variables. It suggests that the learned representations align well with the conditional independence structure of the underlying latent graph in Figs.~\ref{fig:flow} and \ref{fig:pendulum}.

We also observed that the representations in Flow were more entangled compared to Pendulum. There exists an observed but unselected variable  $z_n$ (\textbf{Water Flow}), which introduces additional graph constraints. The graph constraints may conflict with the term enforcing conditional independence, making the learning process more challenging. Addressing this challenge remains an avenue for future work.

\begin{figure}[t]
    \centering
    \includegraphics[width=.95\linewidth]{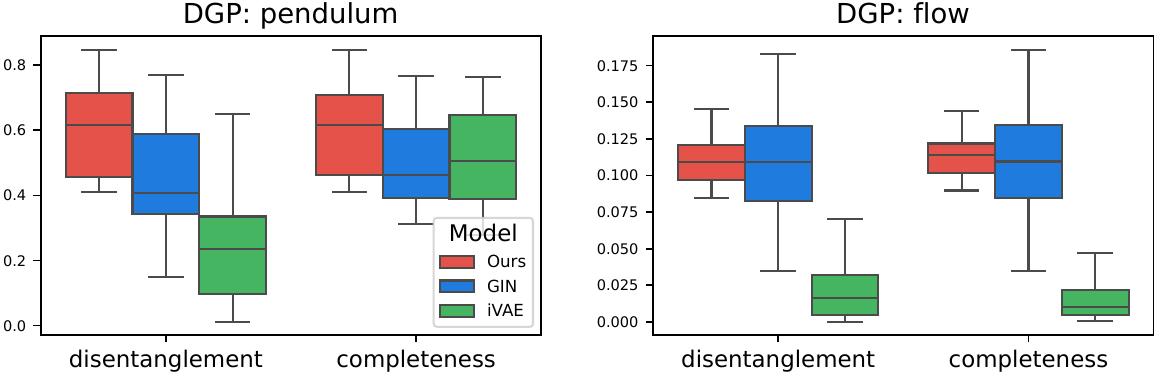}
    \caption{Comparison plot for DCI metric between Ours, GIN, and iVAE on high-dimensional data.}
    \label{fig:dci-highdim}
\end{figure}

\begin{figure}[t]
    \centering
    \includegraphics[width=.95\linewidth]{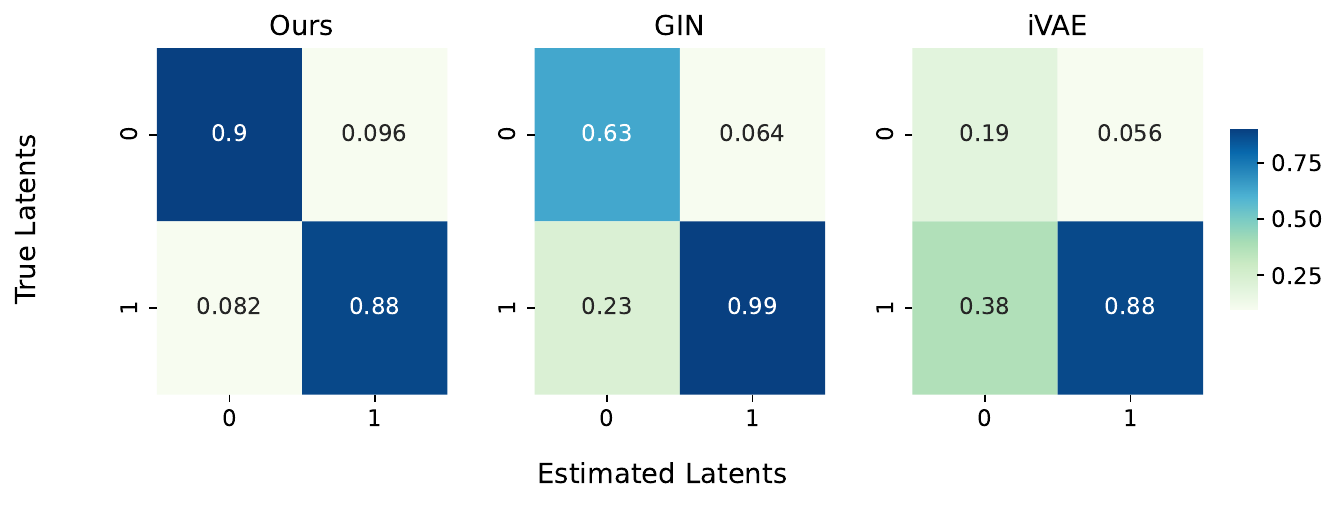}
    \caption{Mean correlation matrices of Ours, GIN applying selection and iVAE applying selection matched with the best permutation on the Pendulum dataset.}
    \label{fig:mcc-pendulum}
\end{figure}

\begin{figure}[t]
    \centering
    \includegraphics[width=.95\linewidth]{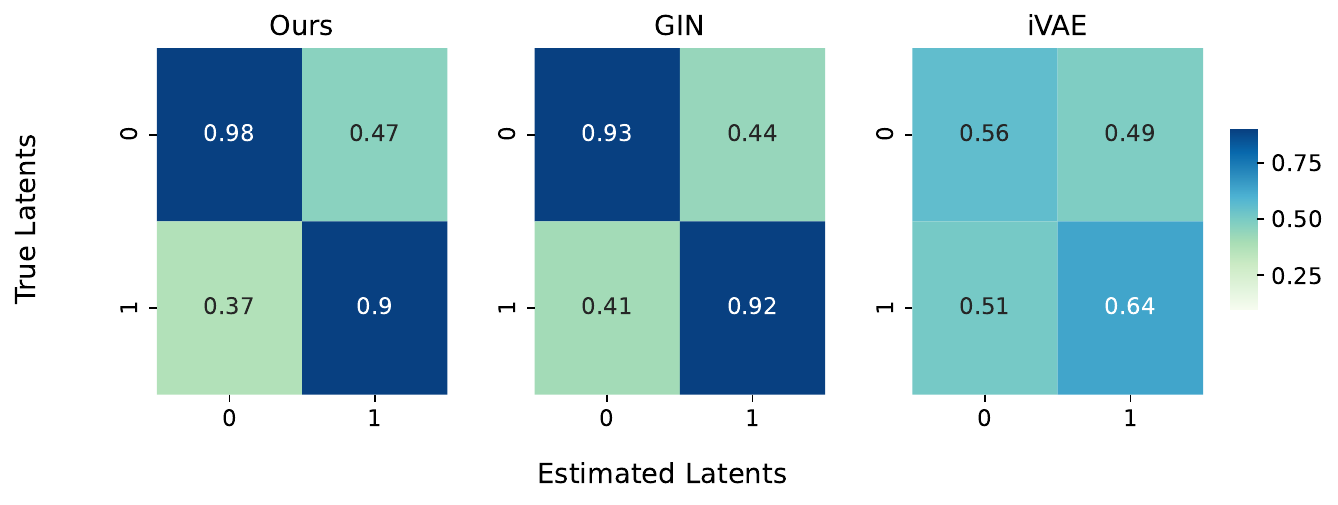}
    \caption{Mean correlation matrices of Ours, GIN applying selection and iVAE applying selection matched with the best permutation on the Flow dataset.}
    \label{fig:mcc-flow}
\end{figure}

\subsection{Latent Traverse}
\begin{figure}[t]
    \centering
    \includegraphics[width=.95\linewidth]{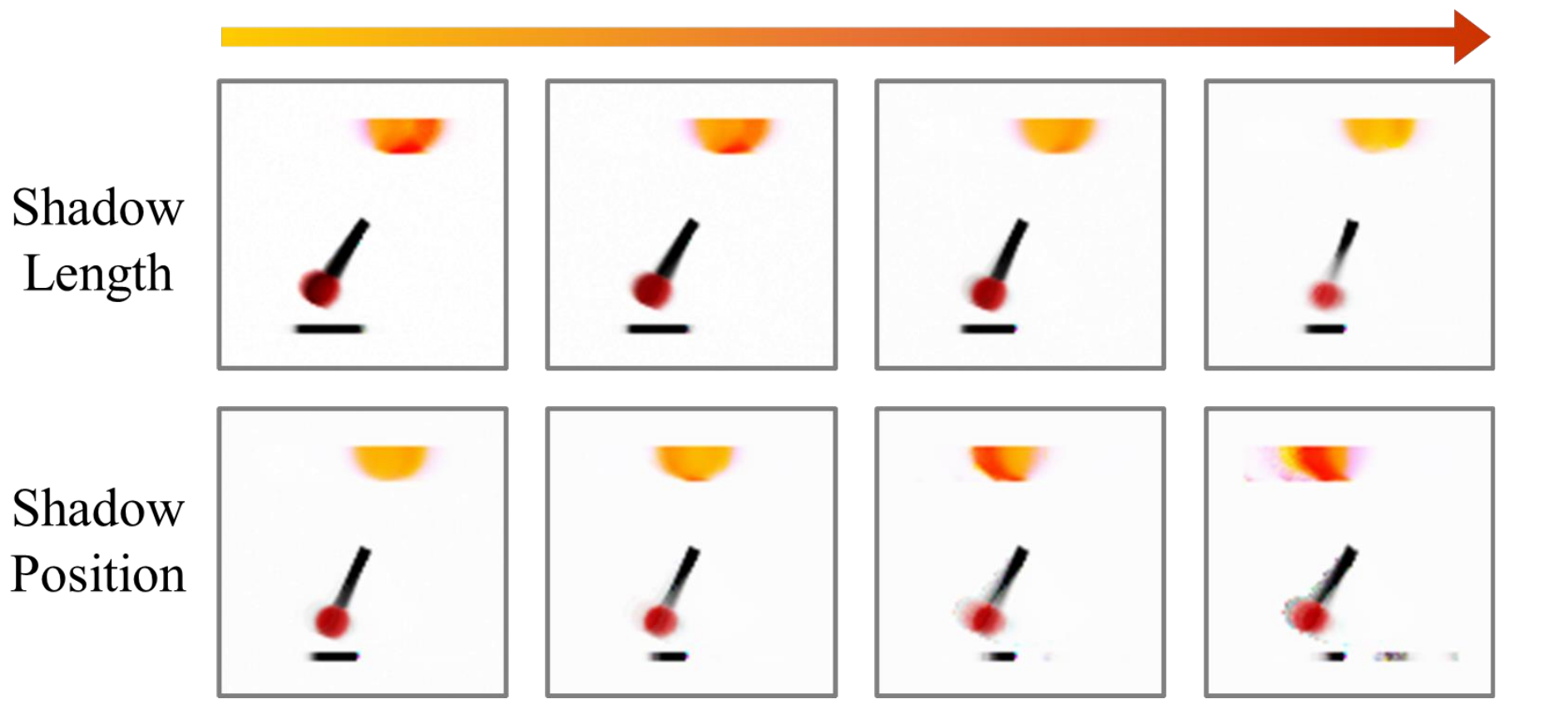}
    \caption{Latent traversal results for unobserved variables. The upper and lower rows show reconstructed images by traversing the variables for shadow length and shadow position, respectively.}
    \label{fig:latent-traverse-pendulum}
\end{figure}

For better comprehensibility, we further extend our model to the image reconstruction task and perform latent traversal to assess whether the factors have been disentangled effectively. We conducted experiments on the pendulum dataset as shown in Fig.~\ref{fig:pendulum}, choosing the pendulum angle and light position as selected variables. To efficiently extract relevant features from high-dimensional image data and visualize disentangled factors, an extra encoder-decoder architecture with an additional MSE (Mean Squared Error) loss was adopted to ensure successful compression and reconstruction of the images.

The encoder initially compresses the image into exogenous latent variables corresponding to the number of nodes in the causal graph. This set of variables is then passed through our model, generating endogenous latent variables of the same dimensionality. The decoder follows a scene-mixture approach, where each scalar from the endogenous latent variables passes through fully connected layers to generate a full-size image. The final output is then reconstructed by averaging these images.


Fig. \ref{fig:latent-traverse-pendulum} presents the results of generating counterfactual images by traversing unobserved latent variables after training our model with the reconstruction objective. As shown in the upper row, traversing the variable associated with shadow length gradually decreases its extent in the reconstructed images. Similarly, modifying the latent variable corresponding to shadow position causes the shadow to shift progressively to the right while mostly preserving other factors. The successful disentanglement of unobserved latent variables further demonstrates the model’s effectiveness in its transferability.

%% file: figs/fig2.tex
\begin{figure}[t]\centering
\begin{subfigure}[b]{0.45\linewidth}\centering
    \begin{tikzpicture}[x=1.2cm,y=1.2cm,scale=0.65,transform shape,font=\small\sffamily]
        \node[draw, circle, fill=betterblue!50, text=white, minimum size=7mm] (p1) at (0, 1.5) {1};
        \node[draw, circle, fill=betterblue!50, text=white, minimum size=7mm] (p2) at (1.5, 1.5) {2};
        \node[draw, circle, minimum size=7mm] (p3) at (-0.3, 0) {3};
        \node[draw, circle, minimum size=7mm] (p4) at (1.8, 0) {4};

        \draw[->] (p1) -- (p3);
        \draw[->] (p1) -- (p4);
        \draw[->] (p2) -- (p3);
        \draw[->] (p2) -- (p4);

        \node[above=0.2cm of p1, xshift=-4mm] {{Pendulum Angle}};
        \node[above=0.2cm of p2, xshift=4mm] {{Light Position}};
        \node[below=0.2cm of p3] {{Shadow Length}};
        \node[below=0.2cm of p4] {{Shadow Position}};
    \end{tikzpicture}
    \caption{Pendulum System}
    \label{fig:pendulum}
\end{subfigure}
\hspace{0.1cm}
\begin{subfigure}[b]{0.45\linewidth}\centering
    \begin{tikzpicture}[x=1.2cm,y=1.2cm,scale=0.65,transform shape,font=\small\sffamily]
        \node[draw, circle, fill=betterblue!50, text=white, minimum size=7mm] (f1) at (0, 1.5) {1};
        \node[draw, circle, minimum size=7mm] (f2) at (1.5, 1.5) {2};
        \node[draw, circle, minimum size=7mm] (f3) at (-0.3, 0) {3};
        \node[draw, circle,fill=betterblue!50, text=white, minimum size=7mm] (f4) at (1.8, 0) {4};

        \draw[->] (f1) -- (f2);
        \draw[->] (f1) -- (f3);
        \draw[->] (f3) -- (f4);
        \draw[->] (f2) -- (f4);

        \node[above=0.2cm of f1] {{Ball Size}};
        \node[above=0.2cm of f2] {{Hole}};
        \node[below=0.2cm of f3] {{Water Height}};
        \node[below=0.2cm of f4] {{Water Flow}};
    \end{tikzpicture}
    \caption{Flow System}
    \label{fig:flow}
\end{subfigure}
\caption{Latent causal graphs for two systems. Colored nodes are observed sources: (a) Pendulum and (b) Flow.}
\label{fig:causal_graphs}
\end{figure}

%% file: sections/AAAI_2025/06-conclusion.tex
Causal representation learning (CRL) offers a promising approach for uncovering the latent variables underlying real-world systems, with significant applications in various domains such as healthcare, economics, and social sciences. Despite challenges in identifiability and the need for strict assumptions, recent advances in CRL have shown potential in addressing these limitations. Our work contributes to the field by being the first to achieve identifiability with observed sources, in which the mixing function includes auxiliary variables. We also propose a framework for selecting and exploiting with observable sources as auxiliary variables to enhance recoverability by exploiting the underlying causal structure of data generating process. Empirical results underscore that the framework is effective for identifying the latents. Furthermore, comparisons between our architecture and other approaches utilizing auxiliary variables on observational data demonstrate that our method outperforms others in identifying true latent variables, effectively mitigating spurious correlations arising from observable sources.

%% file: sections/AAAI_2025/07-acknowledgement.tex
This work was supported
by the IITP (RS-2022-II220953/25\%) and NRF (RS-2023-00211904/50\%, RS-2023-00222663/25\%) grant funded by
the Korean government. 

%% file: sections/AAAI_2025/90-appendix.tex
\section{Discussion}\label{apx:discussion}

\paragraph{Related Work} To deal with the case of \ref{fig1:graph_c}, we can partition the latent sources into conditionally independent sets, 
$\rvz_{c_{i}} (i=1,...,d)$ where $\cup_{i=1}^{d} \rvz_{c_{i}} = \{z_1,\dots,z_n\}$.
It enables a more general formulation of \ref{eq:source-factorization-independent} as \citet{zheng2023generalizing}:
\begin{equation}\label{eq:source-factorization-zheng}
p_{\rvz | \rvu}(\rvz | \rvu) = \prod_{i=1}^{n_i} p_{z_i}(z_i) \prod_{j=1}^{d} p_{\rvz_{c_j}| \rvu}(\rvz_{c_j}| \rvu).
\end{equation}
where $n_i$ is the number of mutually independent sources. \citet{zheng2023generalizing} partition all the sources into a set of mutually independent sources $\rvz_I$ and a set of variables in which do not need to be independent $\rvz_{o^-} = \cup_{i=1}^{d} \rvz_{c_{i}}$. 
In \ref{eq:source-factorization-partial}, we further generalize \ref{eq:source-factorization-zheng} into the setting with observed sources, which includes \ref{eq:source-factorization-zheng} as a special case in that $\rvu$ is independent from DGP.
\begin{equation}
p_{\rvz_{o^-} | \rvz_o}(\rvz_{o^-} | \rvz_o) = \prod_{i=1}^{n_i} p_{z_i}(z_i) \prod_{j=1}^{d} p_{\rvz_{c_j}| \rvz_o}(\rvz_{c_j}| \rvz_o).
\end{equation}
where $\rvz_o$ is observed sources and $\rvz_{o^-}$ is unobserved sources. The former term corresponds to the case without auxiliary variables, which is beyond the scope of our study and thus not considered further.

\paragraph{Metrics} After training the proposed method, we measure Disentanglement, Completeness, Informativeness (DCI) metric \citep{eastwood2018a} to measure the degree of identifiability based on Mean Correlation Coefficient (MCC) matrix which is a widely accepted metric in the literature for
measuring the degree of identifiability \citep{da26a48db8f94dc4b880adcaae51a28e}.
Specifically,
MCC metric is expressed as follows:
\[
\text{MCC}(\rvz, \hat{\rvz}) = \frac{1}{n} \max_{\sigma \in S_n} \sum_{i=1}^{n} \text{corr}(z_i, \hat{z}_{\sigma(i)}),
\]
where $\sigma \in S_n$ is a permutation of the set of indices.
If the model successfully recovers the latent variables, MCC will match estimation with the most similar distributions to true latent (i.e., the highest correlation). 
\begin{figure}
    \centering
    \includegraphics[width=.6\linewidth]{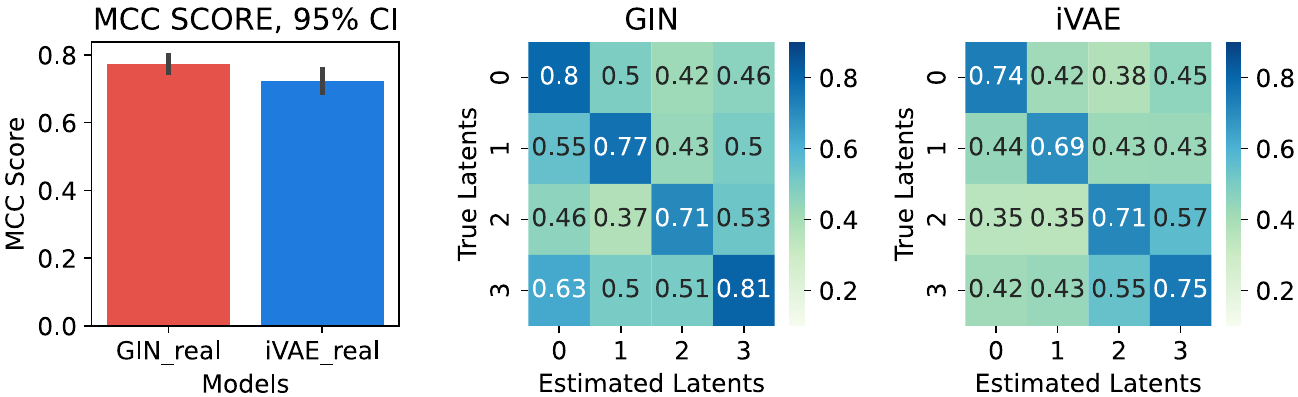}
    \caption{MCC score and mean correlation matrices of GIN and iVAE matched with the best permutation on the setting of Fig \ref{fig1:graph_c}.}
    \label{fig:mcc-example}
\end{figure}
However, the MCC metric alone is insufficient for measuring the degree of identifiability in scenarios involving partially observable sources since 
spurious correlation can arise without disentangling the information of $\rvz_o$ due to the information from the auxiliary variable $\rvz_o$ being entangled with the observation $\rvx$.

The \ref{fig:mcc-example} demonstrates the insufficiency of MCC score in  evaluating the degree of identifiability. The MCC scores of the GIN and iVAE models are around 0.7, suggesting that they recover the true latents reasonably well. However, examining the correlation matrix reveals that the estimated latents also show high correlations with dimensions other than the one with the highest correlation.
This is because existing methods do not account for cases where the mixing function includes auxiliary variables, leading to information from the auxiliary variables being entangled in the estimated latents.

Accordingly, we leverage the DCI metric \citep{eastwood2018a} to evaluate whether the learned representation correctly models the conditional independence structure of the graph without spurious correlation.
The DCI metric evaluates the performance of disentanglement, completeness, and informativeness of representation by measuring the entropy of the importance matrix (in our case, a MCC matrix) If the true sources are well identified without spurious correlation, the representation will be highly disentangled with complete information.

\paragraph{Distinction Between Observed Sources $\rvz_o$ and Observations $\rvx$}
A clear distinction between observed latent sources and observations is fundamental to our framework. The low-level observation represents the raw, directly measurable data, which is the final output of a complex generative process. This process involves the mixing of all underlying latent sources, encompassing both observed and unobserved components.

In contrast, $\rvz_o$ refers to higher-level causal factors or features that are themselves latent but can be either directly observed or reliably inferred from $\rvx$. These $\rvz_o$ variables are considered integral parts of the underlying causal system that our method aims to disentangle. For instance, in an image dataset, the image itself is $\rvx$. Within that image, the joint angles of a robot arm or the position of a light source in a pendulum system, if they can be extracted or measured, would serve as $\rvz_o$. Essentially, $\rvz_o$ represents the "causal handles" that we have some direct access to within the broader set of unobserved latent variables that contribute to $\rvx$.

\paragraph{Identifiability in Collider Case}
In theory, the latent variables $(z_1, z_3)$ and $(z_2, z_3)$ in \ref{fig1:graph_d} should be identified as distinct clusters, reflecting the underlying causal dependencies. To address this, we explicitly incorporate a graphical constraint that enforces the dependencies $z_1 \rightarrow z_3$ and $z_2 \rightarrow z_3$ during training. This constraint encourages the model to capture the correct latent structure even in the presence of colliders. Notably, if $z_3$ were partially observed or served as an auxiliary label, the graphical constraint would further facilitate disentanglement by grounding the latent representation in observable supervision. This highlights the potential of leveraging partial labels or domain knowledge to enhance identifiability in complex latent graphs.
\section{Theoretical Analysis}\label{appx:id}
Firstly, we begin with the definition of identifiability, which is the goal of nonlinear ICA and causal representation learning.
By adopting a Structural Causal Model (SCM, \cite{reason:Pearl09a}), we represent a data-generating process regarding latent sources as
\begin{equation}
    z_i = f_i(\Pa{z_i}, \epsilon_i), \quad \epsilon_i \sim p_{\epsilon_i}, 
\end{equation}
for all $i \in [n]$ 
where $\Pa{\cdot}$ represents parent nodes on a latent causal graph $\mathcal{G}$ consisting of nodes $V$ and edges $E$.
\begin{definition}\label{def1:id} (Identifiability).
Suppose the observations $\*{x}$ are generated by true latent mechanism specified by $\Theta = (\*{f},p(\bm{\epsilon}),\*{g})$ given in \ref{eqn1:generation} and \ref{eqn2:source dep}. The learned generative model parameterized by $\hat{\Theta} = \parens[\big]{\hat{\*{f}},\hat{p}(\bm{\epsilon}),\hat{\*{g}}}$ is observationally equivalent to the true model if the model distribution $p_{\hat{\Theta}}(\*x)$ matches the data distribution $p_{\Theta}(\*{x})$ for any value of $\*{x}$.
Let $A$ be an arbitrary invertible transformation. We say that the model is identifiable up to $A$ if
\begin{equation}
    p_{\hat{\Theta}}(\*{x}) = p_{\Theta}(\*{x}) \implies \hat{\*{g}} = \*{g} \circ A.  
\end{equation}
\end{definition}
Once the mixing function $g$ is identified, the latent variables can be identified up to $A$:
\begin{align*}
    \hat{\mathbf{z}} = \hat{\mathbf{g}}^{-1}(\mathbf{x})
    &= (A^{-1} \circ \mathbf{g}^{-1})(\mathbf{x})  \\
    &= A^{-1}(\mathbf{g}^{-1}(\mathbf{x})) \\
    &= A^{-1}(\mathbf{z}).
\end{align*}

\input{proposition/proposition-aux-aaai}

\begin{proof}
Let $h: \rvz_{o^-} \rightarrow \hat{\rvz}_{o^-}$ denote the transformation from true sources to estimated sources. Thus, we can derive $\hat{g} = g \circ h^{-1}(\rvz_{o^-})$ equivalently as
\begin{align*}
    \mathbf{J}_g(\rvz_{o^-}) = \mathbf{J}_{\hat{g} \circ h}(\rvz_{o^-}) = \mathbf{J}_{\hat{g}}(\hat{\rvz}_{o^-}) \mathbf{J}_{h}(\rvz_{o^-})
\end{align*}
by using chain rule repeatedly.  $\mathbf{J}_{h}(\rvz_{o^-})$ must be invertible and have a non-zero determinant because $\mathbf{J}_{\hat{g}}(\hat{\rvz}_{o^-})$ and $\mathbf{J}_{g}(\rvz_{o^-})$ have full column rank. The change of variable rule and Assumption 2 and 3 make the following equations hold:
\begin{align*}
     p(\rvz_{o^-}\mid\rvz_o)
     (\hat{\rvz}_{o^-}))| &=  p(\hat{\rvz}_{o^-}\mid \rvz_o).
\intertext{By taking logarithm on both sides, we can obtain}
     \log{p(\rvz_{o^-}\mid\rvz_o)} + \log{|\text{det}(\mathbf{J}_{h^{-1}}(\hat{\rvz}_{o^-}))|} &=  \log{p(\hat{\rvz}_{o^-}\mid \rvz_o)}.
\intertext{According to the Assumption 1\footnote{We only consider case of $n_i = 0$ because the identifiability of latter term in \ref{eq:source-factorization-partial} is our focus.} and $\cup_i \rvz_{c_i} = \rvz\setminus\rvz_o$, the joint log densities can be factorized as}
     \sum_{j=c_1}^{c_d}\log{p(\rvz_{j}\mid\rvz_o)} + \log{|\text{det}(\mathbf{J}_{h^{-1}}(\hat{\rvz}_{o^-}))|} &=  \sum_{j=c_1}^{c_d}\log{p(\hat{\rvz}_{j}\mid \hat{\rvz}_o)}.
\end{align*}

Thus, for $\rvz_o = \rvz_{o_0},\dots \rvz_{o_{2d}}$, we have $2d+1$ equations. Subtracting each equation corresponding to $\rvz_{o_1},\dots,\rvz_{o_2d}$ with the equation corresponding to $\rvz_{o_0}$ results in $2d$ equations:
\begin{align}\label{eq:substract}
     \sum_{i=c_1}^{c_d}(\log{p(\rvz_{i}\mid\rvz_{o_j})} - \log{p(\rvz_{i}\mid\rvz_{o_0})} ) =\sum_{i=c_1}^{c_d}(\log{p(\hat{\rvz}_{i}\mid\rvz_{o_j})} - \log{p(\hat{\rvz}_{i}\mid\rvz_{o_0})} )
\end{align}
Take the derivatives of both sides of \ref{eq:substract} with respect to $\hat{z}_k$ and $\hat{z}_v$ where $k,v \in \{1,\dots,n\}$ and they are not indices of the same subspace. 
It is clear that the RHS of \ref{eq:substract} equals to zero because $k$ and $v$ are not indices of the same subspace. 
For the i-th term of the summation on the LHS, we can get following equations:
\begin{align}\label{eq:linear-system-aux}
\sum_{l=i^{(l)}}^{i^{(h)}} \left( \left( \frac{\partial^2 \log p(\rvz_i \mid \rvz_{o_j})}{(\partial z_l)^2} - \frac{\partial^2 \log p(\rvz_i \mid \rvz_{o_0})}{(\partial z_l)^2} \right) \cdot \frac{\partial z_l}{\partial \hat{z}_k} \frac{\partial z_l}{\partial \hat{z}_v} \right. 
+ \left. \left( \frac{\partial \log p (\rvz_i \mid \rvz_{o_j})}{\partial z_l} - \frac{\partial \log p (\rvz_i \mid \rvz_{o_0})}{\partial z_l} \right) \cdot \frac{\partial^2 z_l}{\partial \hat{z}_k \partial \hat{z}_v} \right) = 0,
\end{align}
where $i_l$ and $i_h$ are the minimum and maximum indices of elements in $\rvz_i = (z_{il},\dots, z_{ih})$.
By iterating $i$ from $c_1$ to $c_d$, we can also iterate $l$ from 0 to $n$. Thus, there exists a linear system with a $2d \times 2d$ coefficient matrix.

Considering Assumption 3, the coefficient matrix of the linear system has full rank. The only solution of \ref{eq:linear-system-aux} is $\frac{\partial z_l}{\partial \hat{z}_k} \frac{\partial z_l}{\partial \hat{z}_v} = 0$ and $\frac{\partial^2 z_l}{\partial \hat{z}_k \partial \hat{z}_v}=0$. Note that $\frac{\partial z_l}{\partial \hat{z}_k}$ and $\frac{\partial z_l}{\partial \hat{z}_v}$ cannot be both zero because of invertibility of $h$. Therefore, $k$ can
only be the index of an estimated source from one independent subspace, which, together with the invertibility, leads to the conclusion that $\rvz_{o^-}$ is a composition of an invertible subspace-wise
transformation and a subspace-wise permutation of $\hat{Z}_D$ . So it is the mapping from $\hat{\rvz}_{o^-}$ to $\rvz_{o^-}$ since the
subspace-wise transformation is invertible and the inverse of a block-wise permutation matrix is still a block-wise invertible matrix. 
\end{proof}

We now establish identifiability in the presence of partially observable sources, where an auxiliary variable directly influences the observation $\rvx$ through the mixing function. This constitutes the proof of \ref{prop:partial}. 

\begin{proof} Assume observational equivalence between estimated and true model, i.e. $p_g(\rvx)=p_{\hat g}(\rvx)$. The change of variable rule makes following equations to hold:
\begin{align*}
   p(\rvx)=p(\rvz)\cdot|\text{det}(\mathbf{J}_{g^{-1}})(\rvx)| = p(\hat\rvz)\cdot|\text{det}(\mathbf{J}_{\hat{g}^{-1}})(\rvx)|
\end{align*}
Since $p(\rvz) = p(\rvz_{o^-}\mid\rvz_o)\cdot p(\rvz_o)$,
\begin{align*}
     p(\rvz_{o^-}\mid\rvz_o)\cdot p(\rvz_o)\cdot|\text{det}(\mathbf{J}_{g^{-1}})(\rvx)| =  p(\hat{\rvz}_{o^-}\mid \hat{\rvz}_o)\cdot p(\hat{\rvz}_o)\cdot|\text{det}(\mathbf{J}_{\hat g^{-1}})(\rvx)|
\end{align*}
also can hold.
Note that $p(\hat{\rvz}_o)$ can be replaced by $p(\rvz_o)$ because $\rvz_o$ is already observed.
\begin{align*}
     p(\rvz_{o^-}\mid\rvz_o)\cdot|\text{det}(\mathbf{J}_{g^{-1}})(\rvx)| =  p(\hat{\rvz}_{o^-}\mid \rvz_o)\cdot|\text{det}(\mathbf{J}_{\hat g^{-1}})(\rvx)|
\end{align*}
By taking logarithm on both sides, we can obtain
\begin{align*}
     \log{p(\rvz_{o^-}\mid\rvz_o)} + \log{|\text{det}(\mathbf{J}_{g^{-1}})(\rvx)|} =  \log{p(\hat{\rvz}_{o^-}\mid \rvz_o)} + \log{|\text{det}(\mathbf{J}_{\hat g^{-1}})(\rvx)|}.
\end{align*}
According to the Assumption 1, Assumption 2 and $\cup_i \rvz_{c_i} = \rvz\setminus\rvz_o$, the joint log densities can be factorized as
\begin{align*}
     \sum_{j=c_1}^{c_d}\log{p(\rvz_{j}\mid\rvz_o)}  =  \sum_{j=c_1}^{c_d}\log{p(\hat{\rvz}_{j}\mid \rvz_o)}.
\end{align*}
Thus, for $\rvz_o = \rvz_{o_0},\dots \rvz_{o_{2d-1}}$, we have $2d$ equations.
Take the derivatives of both sides of above equation with respect to $\hat{z}_k$ and $\hat{z}_v$ where $k,v \in \{1,\dots,n\}$ and they are not indices of the same subspace. 
It is clear that the RHS of \ref{eq:substract} equals to zero because $k$ and $v$ are not indices of the same subspace. 
For the i-th term of the summation on the LHS, we can get following equations:
        \begin{align}\label{eq:linear-system}
\sum_{l=i^{(l)}}^{i^{(h)}} \left( \left( \frac{\partial^2 \log p(\rvz_i \mid \rvz_{o})}{(\partial z_l)^2} \right) \cdot \frac{\partial z_l}{\partial \hat{z}_k} \frac{\partial z_l}{\partial \hat{z}_v} \right. 
+ \left. \left( \frac{\partial \log p (\rvz_i \mid \rvz_{o})}{\partial z_l} \right) \cdot \frac{\partial^2 z_l}{\partial \hat{z}_k \partial \hat{z}_v} \right) = 0,
        \end{align}
where $i_l$ and $i_h$ are the minimum and maximum indices of elements in $\rvz_i = (z_{il},\dots, z_{ih})$.
By iterating $i$ from $c_1$ to $c_d$, we can also iterate $l$ from 0 to $n$. 

Considering Assumption 3, the coefficient matrix of the linear system has full rank. The only solution of \ref{eq:linear-system} is $\frac{\partial z_l}{\partial \hat{z}_k} \frac{\partial z_l}{\partial \hat{z}_v} = 0$ and $\frac{\partial^2 z_l}{\partial \hat{z}_k \partial \hat{z}_v}=0$. Note that $\frac{\partial z_l}{\partial \hat{z}_k}$ and $\frac{\partial z_l}{\partial \hat{z}_v}$ cannot be both zero because of invertibility of $h$. Therefore, $k$ can
only be the index of an estimated source from one independent subspace, which, together with the invertibility, leads to the conclusion that $\rvz_{o^-}$ is a composition of an invertible subspace-wise
transformation and a subspace-wise permutation of $\hat{Z}_D$ . So it is the mapping from $\hat{Z}_D$ to $\rvz_{o^-}$ since the
subspace-wise transformation is invertible and the inverse of a block-wise permutation matrix is still a block-wise invertible matrix. 
\end{proof}

\section{Bayes-ball algorithm}\label{appx:alg}
The best known criterion for conditional independence is \textit{d-separation} \citep{GEIGER1990139}. We want to find clusters with inter-cluster d-connectedness and intra-cluster d-separation.

We exploit \textit{Bayes-ball} algorithm  to examine the conditional independence of two node sets on the given graph $\mathcal{G}$. The Bayes-ball algorithm can be extended to partition graph. It returns a set of nodes dependent to an input node set.
\input{algorithm/alg1_fgc}

\section{Experimental Details}\label{appx:imple}
The implementation of the experiments is based on \citet{liang2023causal}. The following Table~\ref{tab:hyper} are hyperparameters for learning Ours, GIN and iVAE.

\input{table/table1}


\section{Limitations and Future Works}\label{appx:discussion}

\paragraph{Prospect to Relax Known DAG Assumption}
Our framework assumes that the full DAG among latent variables is known, which may limit applicability in real-world settings. Nevertheless, there is meaningful potential to relax this assumption. For instance, recent work such as \citet{ali2024crid} explores causal representation learning under non-Markovian graph structures using interventional data. These approaches, like ours, leverage graphical modeling and suggest promising future directions where partial knowledge or data-driven structure learning \citep{reason:Pearl09a,spirtes2000causation,zhang2008completeness,kocaoglu2019characterization,hwang2023discovery,li2023causal,kim2024causal,hwang2024finegrained} could be integrated into our framework to handle more flexible structural settings.

\paragraph{Volume-Preservation and Dimensionality}
We imposed a volume-preserving constraint on the mixing function, ensuring that the log-determinant term remains constant and does not interfere with the training objective. While this constraint offers a concrete and tractable solution, it is not the only possible approach, and further exploration into alternative constraints or relaxation of this assumption is warranted. Additionally, although volume preservation implies matching dimensions between latent and observed spaces, empirical findings from Appendix C of \citet{yang2022nonlinear} indicate that the GIN framework can perform implicit dimensionality reduction by shrinking the variance of redundant dimensions. This robustness may also extend to our setting, suggesting practical flexibility beyond the strict theoretical requirement.

\paragraph{Partitioning of Unobserved}
The identifiability guarantees rely on the assumption that latent variables can be partitioned into conditionally independent clusters given auxiliary variables. When this assumption fails—such as when all latents form a single dependent cluster—the theoretical conditions for recovery no longer hold, leading to unidentifiability. A more precise characterization of graphical conditions that permit non-trivial recovery in these edge cases is essential for extending the method's applicability. Formalizing such criteria remains an open question and will be pursued in future work.

\paragraph{Connection to Concept-based Methods}
A recent approach, called concept-based representation learning \citep{rajendran2024causal}, aims to recover the high-level human-interpretable concepts, instead of true latent generative factors. We consider extending our framework and results to concept-based representation learning as a promising future direction. Another potentially related literature is concept bottleneck models \citep{koh2020concept,yuksekgonul2022post,jeon2024localityaware,oikarinen2023labelfree,hwang2025blackbox}, where the goal is to build inherently interpretable models that make predictions based on the human interpretable concepts. Considering partially observable sources as labeled concepts, we believe that it would be another promising future direction to incorporate our framework into concept bottleneck models.

%% file: proposition/proposition-aux-aaai.tex
\begin{proposition}\label{prop:auxiliary} (Identifiability with external information)
    Suppose the following assumptions hold:
    \begin{enumerate}
        \item The observed data and sources are generated from \ref{eqn1:generation} and \ref{eq:subspace-factorization-isa}
        \item The mixing function $g$ is volume-preserving, i.e., $|\det(\mathbf{J}_g(\rvz))| = 1$.
        \item The observable sources do not have direct edge into the observation $\rvx$, i.e., $\frac{\partial \rvx}{\partial \rvz_o} = 0$
        \item For every value of $\rvz_D$, there exists $2d+1$ values of $\rvz_o$,
        such that the $2d$ vectors $\rvw(\rvz_D, \rvz_{o_i}) - \rvw(\rvz_D, \rvz_{o_0})$ are linearly independent, where vector $\rvw(Z_D, \rvz_{o_i})$ is defined as follows:
        \begin{align*}
            \rvw(\rvz_D, \rvz_{o_i}) &= \big( \rvv(\rvz_{c_1}, \rvz_{o_i}), \dots, \rvv(\rvz_{c_d}, \rvz_{o_i}), \\
            &\quad \rvv'(\rvz_{c_1}, \rvz_{o_i}), \dots, \rvv'(\rvz_{c_d}, \rvz_{o_i}) \big)
        \end{align*}
        where
        \begin{align*}
        \rvv(\rvz_{c_j}, \rvz_{o_i}) &= \left( \frac{\partial \log p(\rvz_{c_j} \mid \rvz_{o_i})}{\partial z_{c_j}^{(l)}}, \dots, \frac{\partial \log p(\rvz_{c_j} \mid \rvz_{o_i})}{\partial z_{c_j}^{(h)}} \right), \\
        \rvv'(\rvz_{c_j}, \rvz_{o_i}) &= \left( \frac{\partial^2 \log p(\rvz_{c_j} \mid \rvz_{o_i})}{\partial (z_{c_j}^{(l)})^2}, \dots, \frac{\partial^2 \log p(\rvz_{c_j} | \rvz_{o_i})}{\partial (z_{c_j}^{(h)})^2} \right)
        \end{align*} and $\rvz_{c_j} = (z_{c_{j}^{(l)}},\dots,z_{c_{j}^{(h)}})$.
    \end{enumerate}
    Then all components of $\rvz_D$ (i.e., $\rvz_{c_i}$ where $c_i \in \{c_1,\dots,c_d\}$) is identifiable up to a subspace-wise invertible transformation and a subspace-wise permutation.
\end{proposition}

%% file: algorithm/alg1_fgc.tex
\begin{algorithm}[t]
    \begin{algorithmic}[1]
    \STATE \textbf{Input}: Graph $G$, Conditioning Set $C$, Observed Set $O$, Set of nodes $R$
    \STATE \textbf{Output}: Updated set of d-connected nodes $R$
    \STATE Initialize an empty set $V$ FOR visited nodes
    \STATE Initialize an empty queue $Q$
    
    \FOR{each node $n$ in $R$}
        \STATE Add $(n, \text{up})$ to $Q$
    \ENDFOR
    
    \WHILE{$Q$ is not empty}
        \STATE $(node, direction) \gets Q.pop()$
        
        \IF{$node \in V$}
            \STATE \textbf{continue}
        \ENDIF
        \STATE Add $node$ to $V$
        
        \IF{$node \in C$ \textbf{and} $direction \neq \text{down}$}
            \STATE \textbf{continue}
        \ENDIF
        
        \IF{$direction = \text{up}$}
            \FOR{each $parent$ of $node$ in $G$}
                \STATE Add $(parent, \text{up})$ to $Q$
            \ENDFOR
            \FOR{each $child$ of $node$ in $G$}
                \STATE Add $(child, \text{down})$ to $Q$
            \ENDFOR
        \ELSIF{$direction = \text{down}$}
            \STATE Initialize $check \gets \text{false}$
            \FOR{each descendant $d$ of $node$ in $G$}
                \IF{$d \in C$}
                    \STATE $check \gets \text{true}$
                    \STATE \textbf{break}
                \ENDIF
            \ENDFOR
            
            \IF{$node \in C$ \textbf{or} $check = \text{true}$}
                \FOR{each $parent$ of $node$ in $G$}
                    \STATE Add $(parent, \text{up})$ to $Q$
                \ENDFOR
            \ELSE
                \FOR{each $child$ of $node$ in $G$}
                    \STATE Add $(child, \text{down})$ to $Q$
                \ENDFOR
            \ENDIF
        \ENDIF

        \IF{$node \notin C$ and $node \notin O$}
            \STATE Add $node$ to $R$
        \ENDIF
    \ENDWHILE
    
    \STATE \textbf{return} $R$
    \end{algorithmic}
\caption{Bayes Ball Algorithm for d-connected nodes}\label{alg:bayes_ball}
\end{algorithm}

%% file: table/table1.tex
\begin{table}[t]\footnotesize 
    \centering
    \caption{Hyperparameters for different models.}
    \label{table:subtable}
    \begin{tabular}{c}
        \begin{subtable}{0.3\textwidth}
            \centering
            \begin{tabular}{@{}lc@{}}
                \toprule
                \multicolumn{2}{c}{\textbf{Ours}} \\
                \midrule
                LR scheduler & Cosine \\
                Learning rate & 0.01 \\
                Number of flows & 8 \\
                Optimizer & Adam \\
                Batch size & 1024 \\
                Training epochs & 20 \\
                \bottomrule
            \end{tabular}
            \caption{Synthetic data}
        \end{subtable}
        \hfill
        \begin{subtable}{0.3\textwidth}
            \centering
            \begin{tabular}{@{}lc@{}}
                \toprule
                \multicolumn{2}{c}{\textbf{GIN}} \\
                \midrule
                LR scheduler & - \\
                Learning rate & 0.01 \\
                Number of flows & 8 \\
                Optimizer & Adam \\
                Batch size & 1024 \\
                Training epochs & 20 \\
                \bottomrule
            \end{tabular}
            \caption{Synthetic data}
        \end{subtable}
        \hfill
        \begin{subtable}{0.3\textwidth}
            \centering
            \begin{tabular}{@{}lc@{}}
                \toprule
                \multicolumn{2}{c}{\textbf{iVAE}} \\
                \midrule
                Number of layers & 3 \\
                Learning rate & 0.0001 \\
                Hidden dim & 4096 \\
                Optimizer & Adam \\
                Batch size & 32 \\
                Training epochs & 20 \\
                \bottomrule
            \end{tabular}
            \caption{Synthetic data}
        \end{subtable}
    \end{tabular}

    \vspace{10pt}  

    \begin{tabular}{c}
        \begin{subtable}{0.3\textwidth}
            \centering
            \begin{tabular}{@{}lc@{}}
                \toprule
                \multicolumn{2}{c}{\textbf{Ours}} \\
                \midrule
                LR scheduler & Cosine \\
                Learning rate & 0.001 \\
                Number of flows & 8 \\
                Optimizer & Adam \\
                Batch size & 1024 \\
                Training epochs & 50 \\
                \bottomrule
            \end{tabular}
            \caption{High-dimensional data}
        \end{subtable}
        \hfill
        \begin{subtable}{0.3\textwidth}
            \centering
            \begin{tabular}{@{}lc@{}}
                \toprule
                \multicolumn{2}{c}{\textbf{GIN}} \\
                \midrule
                LR scheduler & - \\
                Learning rate & 0.001 \\
                Number of flows & 8 \\
                Optimizer & Adam \\
                Batch size & 1024 \\
                Training epochs & 40 \\
                \bottomrule
            \end{tabular}
            \caption{High-dimensional data}
        \end{subtable}
        \hfill
        \begin{subtable}{0.3\textwidth}
            \centering
            \begin{tabular}{@{}lc@{}}
                \toprule
                \multicolumn{2}{c}{\textbf{iVAE}} \\
                \midrule
                Number of layers & 3 \\
                Learning rate & 0.0001 \\
                Hidden dim & 4096 \\
                Optimizer & Adam \\
                Batch size & 1024 \\
                Training epochs & 80 \\
                \bottomrule
            \end{tabular}
            \caption{High-dimensional data}
        \end{subtable}
    \end{tabular}
    \label{tab:hyper}
\end{table}

%% file: anonymous-submission-latex-2026.bbl
\begin{thebibliography}{43}
\providecommand{\natexlab}[1]{#1}

\bibitem[{Baumann et~al.(2022)Baumann, Solowjow, Johansson, and Trimpe}]{baumann2022identifying}
Baumann, D.; Solowjow, F.; Johansson, K.~H.; and Trimpe, S. 2022.
\newblock Identifying Causal Structure in Dynamical Systems.
\newblock \emph{Transactions on Machine Learning Research}.

\bibitem[{Eastwood and Williams(2018)}]{eastwood2018a}
Eastwood, C.; and Williams, C. K.~I. 2018.
\newblock A framework for the quantitative evaluation of disentangled representations.
\newblock In \emph{International Conference on Learning Representations}.

\bibitem[{Geiger, Verma, and Pearl(1990)}]{GEIGER1990139}
Geiger, D.; Verma, T.; and Pearl, J. 1990.
\newblock d-Separation: From Theorems to Algorithms.
\newblock In HENRION, M.; SHACHTER, R.~D.; KANAL, L.~N.; and LEMMER, J.~F., eds., \emph{Uncertainty in Artificial Intelligence}, volume~10 of \emph{Machine Intelligence and Pattern Recognition}, 139--148. North-Holland.

\bibitem[{Hwang et~al.(2024)Hwang, Kwak, Choi, Zhang, and Lee}]{hwang2024finegrained}
Hwang, I.; Kwak, Y.; Choi, S.; Zhang, B.-T.; and Lee, S. 2024.
\newblock Fine-Grained Causal Dynamics Learning with Quantization for Improving Robustness in Reinforcement Learning.
\newblock In \emph{Forty-first International Conference on Machine Learning}.

\bibitem[{Hwang et~al.(2023)Hwang, Kwak, Song, Zhang, and Lee}]{hwang2023discovery}
Hwang, I.; Kwak, Y.; Song, Y.-J.; Zhang, B.-T.; and Lee, S. 2023.
\newblock On discovery of local independence over continuous variables via neural contextual decomposition.
\newblock In \emph{Conference on Causal Learning and Reasoning}, 448--472. PMLR.

\bibitem[{Hwang, Pan, and Bareinboim(2025)}]{hwang2025blackbox}
Hwang, I.; Pan, Y.; and Bareinboim, E. 2025.
\newblock From Black-box to Causal-box: Towards Building More Interpretable Models.
\newblock Technical Report R-127, Columbia CausalAI Laboratory.
\newblock Columbia CausalAI Laboratory, Technical Report (R-127).

\bibitem[{Hyv{\"a}rinen et~al.(2009)Hyv{\"a}rinen, Hurri, Hoyer, Hyv{\"a}rinen, Hurri, and Hoyer}]{hyvarinen2009independent}
Hyv{\"a}rinen, A.; Hurri, J.; Hoyer, P.~O.; Hyv{\"a}rinen, A.; Hurri, J.; and Hoyer, P.~O. 2009.
\newblock \emph{Independent component analysis}.
\newblock Springer.

\bibitem[{Hyv{\"a}rinen and Morioka(2016)}]{da26a48db8f94dc4b880adcaae51a28e}
Hyv{\"a}rinen, A.; and Morioka, H. 2016.
\newblock Unsupervised Feature Extraction by Time-Contrastive Learning and Nonlinear ICA.
\newblock In Garnett, R.; Lee, D.; {von Luxburg}, U.; Guyon, I.; and Sugiyama, M., eds., \emph{Advances in Neural Information Processing Systems}, number NIPS 2016 in Advances in neural information processing systems, 3772--3780. United States: Neural Information Processing Systems Foundation.
\newblock Annual Conference on Neural Information Processing Systems, NIPS ; Conference date: 05-12-2016 Through 10-12-2016.

\bibitem[{Hyv\"{a}rinen and Pajunen(1999)}]{10.1016/S0893-6080(98)00140-3}
Hyv\"{a}rinen, A.; and Pajunen, P. 1999.
\newblock Nonlinear independent component analysis: existence and uniqueness results.
\newblock \emph{Neural Netw.}, 12(3).

\bibitem[{Hyvarinen, Sasaki, and Turner(2019)}]{pmlr-v89-hyvarinen19a}
Hyvarinen, A.; Sasaki, H.; and Turner, R. 2019.
\newblock Nonlinear ICA Using Auxiliary Variables and Generalized Contrastive Learning.
\newblock In Chaudhuri, K.; and Sugiyama, M., eds., \emph{Proceedings of the Twenty-Second International Conference on Artificial Intelligence and Statistics}, volume~89 of \emph{Proceedings of Machine Learning Research}, 859--868. PMLR.

\bibitem[{Jeon et~al.(2024)Jeon, Hwang, Lee, and Zhang}]{jeon2024localityaware}
Jeon, S.; Hwang, I.; Lee, S.; and Zhang, B.-T. 2024.
\newblock Locality-aware Concept Bottleneck Model.
\newblock In \emph{UniReps: 2nd Edition of the Workshop on Unifying Representations in Neural Models}.

\bibitem[{Khemakhem et~al.(2020)Khemakhem, Kingma, Monti, and Hyvarinen}]{pmlr-v108-khemakhem20a}
Khemakhem, I.; Kingma, D.; Monti, R.; and Hyvarinen, A. 2020.
\newblock Variational Autoencoders and Nonlinear ICA: A Unifying Framework.
\newblock In Chiappa, S.; and Calandra, R., eds., \emph{Proceedings of the Twenty Third International Conference on Artificial Intelligence and Statistics}, volume 108 of \emph{Proceedings of Machine Learning Research}, 2207--2217. PMLR.

\bibitem[{Kim, Hwang, and Lee(2024)}]{kim2024causal}
Kim, J.; Hwang, I.; and Lee, S. 2024.
\newblock Causal Discovery with Deductive Reasoning: One Less Problem.
\newblock In \emph{Uncertainty in Artificial Intelligence}, 1999--2017. PMLR.

\bibitem[{Kocaoglu et~al.(2019)Kocaoglu, Jaber, Shanmugam, and Bareinboim}]{kocaoglu2019characterization}
Kocaoglu, M.; Jaber, A.; Shanmugam, K.; and Bareinboim, E. 2019.
\newblock Characterization and learning of causal graphs with latent variables from soft interventions.
\newblock \emph{Advances in neural information processing systems}, 32.

\bibitem[{Koh et~al.(2020)Koh, Nguyen, Tang, Mussmann, Pierson, Kim, and Liang}]{koh2020concept}
Koh, P.~W.; Nguyen, T.; Tang, Y.~S.; Mussmann, S.; Pierson, E.; Kim, B.; and Liang, P. 2020.
\newblock Concept bottleneck models.
\newblock In \emph{International conference on machine learning}, 5338--5348. PMLR.

\bibitem[{K{\"u}gelgen et~al.(2021)K{\"u}gelgen, Sharma, Gresele, Brendel, Sch{\"o}lkopf, Besserve, and Locatello}]{gelgen2021selfsupervised}
K{\"u}gelgen, J.~V.; Sharma, Y.; Gresele, L.; Brendel, W.; Sch{\"o}lkopf, B.; Besserve, M.; and Locatello, F. 2021.
\newblock Self-Supervised Learning with Data Augmentations Provably Isolates Content from Style.
\newblock In Beygelzimer, A.; Dauphin, Y.; Liang, P.; and Vaughan, J.~W., eds., \emph{Advances in Neural Information Processing Systems}.

\bibitem[{Li, Jaber, and Bareinboim(2023)}]{li2023causal}
Li, A.; Jaber, A.; and Bareinboim, E. 2023.
\newblock Causal discovery from observational and interventional data across multiple environments.
\newblock \emph{Advances in Neural Information Processing Systems}, 36: 16942--16956.

\bibitem[{Li, Pan, and Bareinboim(2024)}]{ali2024crid}
Li, A.; Pan, E.; and Bareinboim, E. 2024.
\newblock Disentangled Representation Learning in Non-Markovian Causal Systems.
\newblock Technical Report R-110, Causal Artificial Intelligence Lab, Columbia University.

\bibitem[{Liang et~al.(2023)Liang, Keki{\'c}, von K{\"u}gelgen, Buchholz, Besserve, Gresele, and Sch{\"o}lkopf}]{liang2023causal}
Liang, W.; Keki{\'c}, A.; von K{\"u}gelgen, J.; Buchholz, S.; Besserve, M.; Gresele, L.; and Sch{\"o}lkopf, B. 2023.
\newblock Causal Component Analysis.
\newblock In \emph{Thirty-seventh Conference on Neural Information Processing Systems}.

\bibitem[{Lippe et~al.(2023)Lippe, Magliacane, L{\"o}we, Asano, Cohen, and Gavves}]{lippe2023causal}
Lippe, P.; Magliacane, S.; L{\"o}we, S.; Asano, Y.~M.; Cohen, T.; and Gavves, E. 2023.
\newblock Causal Representation Learning for Instantaneous and Temporal Effects in Interactive Systems.
\newblock In \emph{The Eleventh International Conference on Learning Representations}.

\bibitem[{Locatello et~al.(2019)Locatello, Bauer, Lučić, Rätsch, Gelly, Schölkopf, and Bachem}]{47692}
Locatello, F.; Bauer, S.; Lučić, M.; Rätsch, G.; Gelly, S.; Schölkopf, B.; and Bachem, O.~F. 2019.
\newblock Challenging Common Assumptions in the Unsupervised Learning of Disentangled Representations.
\newblock In \emph{International Conference on Machine Learning}.
\newblock Best Paper Award.

\bibitem[{Lu et~al.(2022)Lu, Wu, Hern{\'a}ndez-Lobato, and Sch{\"o}lkopf}]{lu2022invariant}
Lu, C.; Wu, Y.; Hern{\'a}ndez-Lobato, J.~M.; and Sch{\"o}lkopf, B. 2022.
\newblock Invariant Causal Representation Learning for Out-of-Distribution Generalization.
\newblock In \emph{International Conference on Learning Representations}.

\bibitem[{Mooij, Janzing, and Sch\"{o}lkopf(2013)}]{10.5555/3023638.3023683}
Mooij, J.~M.; Janzing, D.; and Sch\"{o}lkopf, B. 2013.
\newblock From ordinary differential equations to structural causal models: the deterministic case.
\newblock In \emph{Proceedings of the Twenty-Ninth Conference on Uncertainty in Artificial Intelligence}, UAI'13, 440–448. Arlington, Virginia, USA: AUAI Press.

\bibitem[{Oikarinen et~al.(2023)Oikarinen, Das, Nguyen, and Weng}]{oikarinen2023labelfree}
Oikarinen, T.; Das, S.; Nguyen, L.~M.; and Weng, T.-W. 2023.
\newblock Label-free Concept Bottleneck Models.
\newblock In \emph{The Eleventh International Conference on Learning Representations}.

\bibitem[{Pan and Bareinboim(2024)}]{pan2024counterfactual}
Pan, Y.; and Bareinboim, E. 2024.
\newblock Counterfactual Image Editing.
\newblock In \emph{Proceedings of the 41st International Conference on Machine Learning (ICML)}.

\bibitem[{Pearl(2009)}]{reason:Pearl09a}
Pearl, J. 2009.
\newblock \emph{Causality: Models, Reasoning and Inference}.
\newblock Cambridge University Press, 2nd edition.

\bibitem[{Rajendran et~al.(2024)Rajendran, Buchholz, Aragam, Sch{\"o}lkopf, and Ravikumar}]{rajendran2024causal}
Rajendran, G.; Buchholz, S.; Aragam, B.; Sch{\"o}lkopf, B.; and Ravikumar, P. 2024.
\newblock From causal to concept-based representation learning.
\newblock \emph{Advances in Neural Information Processing Systems}, 37: 101250--101296.

\bibitem[{Sanchez et~al.(2022)Sanchez, Voisey, Xia, Watson, O’Neil, and Tsaftaris}]{sanchez2022causal}
Sanchez, P.; Voisey, J.~P.; Xia, T.; Watson, H.~I.; O’Neil, A.~Q.; and Tsaftaris, S.~A. 2022.
\newblock Causal machine learning for healthcare and precision medicine.
\newblock \emph{Royal Society Open Science}, 9(8): 220638.

\bibitem[{Schölkopf et~al.(2021)Schölkopf, Locatello, Bauer, Ke, Kalchbrenner, Goyal, and Bengio}]{9363924}
Schölkopf, B.; Locatello, F.; Bauer, S.; Ke, N.~R.; Kalchbrenner, N.; Goyal, A.; and Bengio, Y. 2021.
\newblock Toward Causal Representation Learning.
\newblock \emph{Proceedings of the IEEE}, 109(5): 612--634.

\bibitem[{Shachter(1998)}]{10.5555/2074094.2074151}
Shachter, R.~D. 1998.
\newblock Bayes-ball: Rational pastime (for determining irrelevance and requisite information in belief networks and influence diagrams).
\newblock In \emph{Proceedings of the Fourteenth Conference on Uncertainty in Artificial Intelligence}, UAI'98, 480–487. San Francisco, CA, USA: Morgan Kaufmann Publishers Inc.
\newblock ISBN 155860555X.

\bibitem[{Sorrenson, Rother, and Köthe(2020)}]{sorrenson2020disentanglementnonlinearicageneral}
Sorrenson, P.; Rother, C.; and Köthe, U. 2020.
\newblock Disentanglement by Nonlinear ICA with General Incompressible-flow Networks (GIN).
\newblock arXiv:2001.04872.

\bibitem[{Spirtes, Glymour, and Scheines(2000)}]{spirtes2000causation}
Spirtes, P.; Glymour, C.~N.; and Scheines, R. 2000.
\newblock \emph{Causation, prediction, and search}.
\newblock MIT press.

\bibitem[{Wang, Chen, and Yao(2024)}]{wang2024causally}
Wang, S.; Chen, X.; and Yao, L. 2024.
\newblock On Causally Disentangled State Representation Learning for Reinforcement Learning based Recommender Systems.
\newblock In \emph{Proceedings of the 33rd ACM International Conference on Information and Knowledge Management}, 2390--2399.

\bibitem[{Wang et~al.(2022)Wang, Lin, Feng, He, Lin, and Chua}]{wang2022causal}
Wang, W.; Lin, X.; Feng, F.; He, X.; Lin, M.; and Chua, T.-S. 2022.
\newblock Causal representation learning for out-of-distribution recommendation.
\newblock In \emph{Proceedings of the ACM Web Conference 2022}, 3562--3571.

\bibitem[{Yang et~al.(2021)Yang, Liu, Chen, Shen, Hao, and Wang}]{yang2021causalvae}
Yang, M.; Liu, F.; Chen, Z.; Shen, X.; Hao, J.; and Wang, J. 2021.
\newblock Causalvae: Disentangled representation learning via neural structural causal models.
\newblock In \emph{Proceedings of the IEEE/CVF conference on computer vision and pattern recognition}, 9593--9602.

\bibitem[{Yang et~al.(2024)Yang, Li, Liu, Wang, Lu, and Liu}]{yang2024disentangled}
Yang, X.; Li, X.; Liu, Z.; Wang, Y.; Lu, S.; and Liu, F. 2024.
\newblock Disentangled causal representation learning for debiasing recommendation with uniform data.
\newblock \emph{Applied Intelligence}, 1--16.

\bibitem[{Yang et~al.(2022)Yang, Wang, Sun, Zhang, Zhang, Li, and Yan}]{yang2022nonlinear}
Yang, X.; Wang, Y.; Sun, J.; Zhang, X.; Zhang, S.; Li, Z.; and Yan, J. 2022.
\newblock Nonlinear {ICA} Using Volume-Preserving Transformations.
\newblock In \emph{International Conference on Learning Representations}.

\bibitem[{Yao, Muller, and Locatello(2024)}]{yao2024marrying}
Yao, D.; Muller, C.; and Locatello, F. 2024.
\newblock Marrying Causal Representation Learning with Dynamical Systems for Science.
\newblock \emph{arXiv preprint arXiv:2405.13888}.

\bibitem[{Yao et~al.(2024)Yao, Xu, Lachapelle, Magliacane, Taslakian, Martius, von K{\"u}gelgen, and Locatello}]{yao2024multiview}
Yao, D.; Xu, D.; Lachapelle, S.; Magliacane, S.; Taslakian, P.; Martius, G.; von K{\"u}gelgen, J.; and Locatello, F. 2024.
\newblock Multi-View Causal Representation Learning with Partial Observability.
\newblock In \emph{The Twelfth International Conference on Learning Representations}.

\bibitem[{Yuksekgonul, Wang, and Zou(2022)}]{yuksekgonul2022post}
Yuksekgonul, M.; Wang, M.; and Zou, J. 2022.
\newblock Post-hoc concept bottleneck models.
\newblock \emph{arXiv preprint arXiv:2205.15480}.

\bibitem[{Zhang(2008)}]{zhang2008completeness}
Zhang, J. 2008.
\newblock On the completeness of orientation rules for causal discovery in the presence of latent confounders and selection bias.
\newblock \emph{Artificial Intelligence}, 172(16-17): 1873--1896.

\bibitem[{Zhang et~al.(2024)Zhang, Xie, Ng, and Zheng}]{zhang2024causal}
Zhang, K.; Xie, S.; Ng, I.; and Zheng, Y. 2024.
\newblock Causal Representation Learning from Multiple Distributions: A General Setting.
\newblock In \emph{Forty-first International Conference on Machine Learning}.

\bibitem[{Zheng and Zhang(2023)}]{zheng2023generalizing}
Zheng, Y.; and Zhang, K. 2023.
\newblock Generalizing Nonlinear {ICA} Beyond Structural Sparsity.
\newblock In \emph{Thirty-seventh Conference on Neural Information Processing Systems}.

\end{thebibliography}
